\documentclass[journal]{IEEEtran}
\usepackage[utf8]{inputenc}
\usepackage[pdftex]{graphicx}
\usepackage{epstopdf}
\usepackage[utf8]{inputenc} 
\usepackage[T1]{fontenc}    
\usepackage{hyperref}       
\usepackage{url}            
\usepackage{booktabs}       
\usepackage{amsfonts}       
\usepackage{nicefrac}       
\usepackage{microtype}      
\usepackage{graphics, times, cite}

\usepackage{tikz}
\usepackage[english]{babel}
\usepackage{ifpdf}
\usepackage{url}
\usepackage{bm}
\ifpdf
	\usepackage[pdftex]{graphicx}
	\graphicspath{{./figures/}}
\else
	\usepackage[dvips]{graphicx}
	\graphicspath{./figures/}
\fi
\usepackage{color}
\usepackage{pgf, tikz, pgfplots, epstopdf}
\usetikzlibrary{shapes, arrows, automata}
\usepackage{caption} 
\usepackage{subcaption} 
\usepackage{amsmath}
\usepackage{amsfonts, amssymb, amsthm}
\usepackage{mathrsfs}
\usepackage{upgreek}
\usepackage{algorithm,algpseudocode}
\usepackage{enumerate}
\usepackage{multirow}
\usepackage{rotating}
\usepackage{needspace}

\newcommand{\myparagraph}[1]{\needspace{1\baselineskip}\medskip\noindent {\bf #1.}}

\input{mySymbol.sty}

\usepackage{booktabs}       
\usepackage{nicefrac}       
\usepackage{microtype}      
\usepackage{xcolor}         

\theoremstyle{definition}

\newtheorem{assumption}{\hspace{0pt}\bf Assumption\hspace{-0.05cm}}
\newtheorem{lemma}{\hspace{0pt}\bf Lemma}
\newtheorem{proposition}{\hspace{0pt}\bf Proposition}

\newtheorem{theorem}{\hspace{0pt}\bf Theorem}
\newtheorem{corollary}{\hspace{0pt}\bf Corollary}

\newtheorem{remark}{\hspace{0pt}\bf Remark}

\newtheorem{definition}{\hspace{0pt}\bf Definition}

\newcommand{\QED}{\hfill\ensuremath{\blacksquare}}

\title{Stability to Deformations of Manifold Filters and \\ Manifold Neural Networks}

\author{Zhiyang Wang \quad Luana Ruiz \quad Alejandro Ribeiro\thanks{Supported by NSF-Simons MoDL, NSF AI Institutes program and NSF HDR TRipods. ZW, AR are with the Department of Electrical and Systems Engineering, University of Pennsylvania, PA, email: \{zhiyangw, aribeiro\}@seas.upenn.edu. LR is with the Department of Applied Mathematics and Statistics, Johns Hopkins University, MD, email: lrubini1@jh.edu. Preliminary results presented in \cite{wang2021stability}\cite{wang2022stability}\cite{wang2022convolutional}.} }
\begin{document}

\maketitle

\begin{abstract}
The paper defines and studies manifold (M) convolutional filters and neural networks (NNs). \emph{Manifold} filters and MNNs are defined in terms of the Laplace-Beltrami operator exponential and are such that \emph{graph} (G) filters and neural networks (NNs) are recovered as discrete approximations when the manifold is sampled. These filters admit a spectral representation which is a generalization of both the spectral representation of graph filters and the frequency response of standard convolutional filters in continuous time. The main technical contribution of the paper is to analyze the stability of manifold filters and MNNs to smooth deformations of the manifold. This analysis generalizes known stability properties of graph filters and GNNs and it is also a generalization of known stability properties of standard convolutional filters and neural networks in continuous time. The most important observation that follows from this analysis is that manifold filters, same as graph filters and standard continuous time filters, have difficulty discriminating high frequency components in the presence of deformations. This is a challenge that can be ameliorated with the use of manifold, graph, or continuous time neural networks. The most important practical consequence of this analysis is to shed light on the behavior of graph filters and GNNs in large scale graphs.
\end{abstract}


\section{Introduction}
\label{sec:intro}

Graph convolutional filters\cite{ortega2018graph, isufi2016autoregressive, gama2020graphs} and graph neural networks (GNNs) \cite{scarselli2008graph, gama2019convolutional, zhou2020graph} have become the tool of choice for signal and information processing on graphs, e.g., \cite{fan2019graph, wu2019session, chowdhury2021unfolding, wang2020unsupervised}. In several applications, graphs can be considered as samples of a manifold. This is sometimes quite explicit as in the case of, e.g., point clouds \cite{bronstein2017geometric, bronstein2021geometric} and sometimes somewhat implicit as in the case of, e.g., wireless communication networks \cite{fan2019graph, wu2019session}. In this context we can think of convolutions on graphs and GNNs as discretizations of their manifold counterparts. This is important because one can often gain valuable insights about discrete (graph) signal processing by studying continuous (manifold) signal processing. 

The main technical contribution of this paper is analyzing the stability of manifold filters and manifold neural networks (MNNs) to smooth deformations of the manifold. Prior to developing stability analyses we define \emph{manifold} convolutional filters and frequency responses that are consistent with \emph{graph} convolutional filters and frequency responses. 

\myparagraph{Manifold Filters} Consider a manifold with Laplace-Beltrami (LB) operator $\ccalL$ and formulate the corresponding manifold diffusion equation with respect to an auxiliary time variable $t$. For initial condition $f(x)$, manifold diffusions generate the time varying manifold function $u(x,t) = e^{-t\ccalL}f(x)$ in which $e^{-t\ccalL}$ is the LB operator exponential (Section \ref{sec_lb}). A manifold convolutional filter $\tdh(t)$ then acts on the function $f(x)$ as the integral over time of the product of $\tdh(t)$ with the diffusion sequence $u(x,t)$ (Section \ref{sec_manifold_filters}),
\begin{equation} \label{eqn:manifold-conv-spatial_in_the_intro}
   g(x) = (\bbh f)(x) =\int_0^\infty \tdh(t)e^{-t\ccalL}f(x)\text{d}t  .
\end{equation}
A manifold convolutional filter is such that we can recover graph filters through discretization of the manifold and discretization of the auxiliary time variable $t$ (Section \ref{sec:discre_nn}). The reason for this connection is that graph filters are linear combinations of the elements of the graph diffusion sequence \cite{shuman2013emerging, ortega2018graph, sandryhaila2013discrete} and \eqref{eqn:manifold-conv-spatial_in_the_intro} defines manifold filters as linear combinations of the elements of the manifold diffusion sequence. Manifold filters are also generalizations of standard time convolutions. This requires consideration of the wave equation on the real line so that the exponential of the derivative operator $e^{-t\partial/\partial x}$ appears in \eqref{eqn:manifold-conv-spatial_in_the_intro} (Appendix \ref{app:rem_convolution}).

\myparagraph{Frequency Response of Manifold Filters} We define a filter's frequency response as  
\begin{equation}\label{eqn:operator-frequency_in_the_intro}
    \hat{h}(\lambda)=\int_0^\infty \tdh(t) e^{- t \lambda  }\text{d}t \text{.}
\end{equation}
This definition is motivated by the fact that manifold convolutions can be decomposed on separate spectral components (Section \ref{sec_spectral_representation}). Indeed, let $\bm\phi_i$ be an eigenfunction of the LB operator associated with eigenvalue $\lambda_i$. If $[\hat{f}]_i = \langle f, \bm\phi_i \rangle_{L^2(\ccalM)}$ and $[\hat{g}]_i = \langle g, \bm\phi_i \rangle_{L^2(\ccalM)}$ are projections of the functions $f$ and $g$ in \eqref{eqn:manifold-conv-spatial_in_the_intro} on this eigenfunction, $[\hat{g}]_i$ depends only on $[\hat{f}]_i$ and can be written as $[\hat{g}]_i = \hat{h}(\lambda_i) [\hat{f}]_i$ (Proposition \ref{prop:filter-spectral}). 

This definition of the frequency response of a manifold filter generalizes the frequency response of a standard time filter. This is because the definition in \eqref{eqn:operator-frequency_in_the_intro} is a Laplace transform, which reduces to the Fourier transform of a filter's impulse response when restricted to the imaginary axis $\lambda=j\omega$. It is also a generalization of the frequency response of a graph filter, which is a z-transform \cite{oppenheim1997signals}.

\myparagraph{Stability to Deformations} We establish stability of manifold filters and MNNs with respect to domain deformations formally defined as manifold diffeomorphisms. We consider a signal $f$ defined on a manifold $\ccalM$ and the signal $f\circ\tau$ made up of the composition of $f$ with a diffeomorphism. The signals $f$ and $f \circ \tau$ are passed through the same MNN $\bm\Phi$. In this paper we prove that if $\tau$ is $\epsilon$-small and $\epsilon$-smooth the respective MNN outputs satisfy
\begin{align}\label{eqn_bound_for_intro}
   \| \bm\Phi(f \circ \tau)-\bm\Phi(f)\|_{L^2(\ccalM)}=O(\epsilon)\|f\|_{L^2(\ccalM)},    
\end{align}
provided that the manifold filters in the layers of the MNN satisfy certain spectral properties (Sections \ref{sec_manifold_stability} and \ref{sec:stability_nn}). 

The bound in \eqref{eqn_bound_for_intro} is a generalization of the standard convolutional neural network (CNN) bound in \cite{mallat2012group}, which studies diffeomorphisms of the real line and its effect on convolutional filter banks and CNNs. The bound is also a limit version of the GNN bounds of \cite{gama2020stability} and related literature \cite{gama2019stability, verma2019stability, zou2020graph, ruiz2021graph, ruiz2021graphon}. As is the case of \cite{mallat2012group, gama2020stability, gama2019stability, verma2019stability, zou2020graph, ruiz2021graph, ruiz2021graphon} the bound in \eqref{eqn_bound_for_intro} holds when the frequency response $\hat{h}(\lambda)$ of the filters that make up the layers of the MNN have decreasing variability with increasing $\lambda$ (Section \ref{sec_manifold_stability}). Thus, stability requires \emph{layers} with limited ability to discriminate high frequency components and implies that one may expect \emph{multi-layered} MNNs to outperform filters in learning tasks in which high frequency components are important (Section \ref{subsec:discussion}).

\myparagraph{Related Work and Significance} We focus on the stability analysis of MNNs -- the limit version of GNNs -- because all of the existing GNN stability results have bounds that grow with the number of nodes in the graph \cite{gama2019stability, ruiz2021graph}. To overcome this limitation, many works have studied neural networks on graphons \cite{ruiz2021graphon, maskey2021transferability, ruiz2021transferability} and more general graph models with variable sparsity \cite{keriven2020convergence}. Results in these settings are independent of graph size – same as the results for a limit object presented here. 
Manifolds can provide limit models for relatively sparse graphs, i.e. $\epsilon$-graphs and $k$-NN graphs \cite{calder2019improved}. Even if there exist other random graph models allowing to model moderately sparse graphs (e.g., \cite{keriven2020convergence}), these models do not have the physical interpretation of a manifold, which is often a better descriptor of real-world domains.
Of particular relevance to our paper is the work on GNN transferability for graphs that are sampled from a general topological space \cite{levie2019transferability}.

\myparagraph{Organization} Section \ref{sec:stability_filter} introduces preliminary definitions (Section \ref{sec_lb}), defines manifold convolutions (Section \ref{sec_manifold_filters}), and introduces the spectral domain representation of manifold filters (Section \ref{sec_spectral_representation}). Section \ref{sec_manifold_stability} studies the stability of manifold filters to manifold deformations. It shows that a diffeomorphism results in a perturbation of the LB operator that involves additive and multiplicative terms (Theorem \ref{thm:perturb}). It then goes on to study the effect on manifold filters of additive (Section \ref{subsec:filter-absolute}) and multiplicative (Section \ref{subsec:filter-relative}) perturbations of the LB operator. Section \ref{sec:stability_nn} extends the analysis of manifold filter stability to manifold neural networks and Section \ref{subsec:discussion} discusses the implications of the results derived in Sections \ref{sec_manifold_stability} and \ref{sec:stability_nn}. Section \ref{sec:discre_nn} explains how to recover graph filters and GNNs from the discretization of a manifold. Section \ref{sec:simu} illustrates the results of Sections \ref{sec_manifold_stability} and \ref{sec:stability_nn} with numerical examples. Section \ref{sec:conclusion} concludes the paper. Proofs are deferred to appendices.

\section{Manifold Convolutional Filters}
\label{sec:stability_filter}

Consider a compact, smooth and differentiable $d$-dimensional submanifold $\ccalM \subset \reals^\mathsf{N}$. For simplicity, in this paper we use the words submanifold and manifold interchangeably, assuming the manifold $\ccalM$ to always be embedded in $\reals^\mathsf{N}$. This embedding induces a Riemannian structure \cite{gallier2020differential}.
In turn, the Riemannian structure allows defining a measure $\mu$ over the manifold as well as a notion of \emph{length} for smooth curves on $\ccalM$. Given two points $x, y\in\ccalM$, the length of the shortest curve between $x$ and $y$ is denoted $\text{dist(x,y)}$ and called the geodesic distance between these points.

We consider the manifold $\ccalM$ to be the support of data that we represent as smooth real scalar functions $f:\ccalM\rightarrow \reals$. We call these scalar functions \emph{manifold signals}. We focus on manifold signals that have finite energy, such that $f \in L^2(\ccalM)$. Since  $L^2(\ccalM)$ is a Hilbert space, it is equipped with an inner product given by
\begin{equation}\label{eqn:innerproduct}
    \langle f,g \rangle_{L^2(\ccalM)}=\int_\ccalM f(x)g(x) \text{d}\mu(x) 
\end{equation}
where $\text{d}\mu(x)$ is the $d$-dimensional volume element corresponding to measure $\mu$. Thus, the energy of the signal $f$ is given by $\|f\|^2_{L^2(\ccalM)}={\langle f,f \rangle_{L^2(\ccalM)}}$.


\subsection{Laplace-Beltrami Operator}\label{sec_lb}

On manifolds, differentiation is implemented by the \emph{intrinsic gradient} -- a local operator acting on a neighborhood of each point on the manifold that is homeomorphic to a $d$-dimensional Euclidean space. This neighborhood contains all the vectors tangent to $\ccalM$ at $x$, which is called the tangent space of $x \in \ccalM$ and is denoted as $T_x\ccalM$. The disjoint union of all tangent spaces on $\ccalM$ is defined as the tangent bundle $T\ccalM$. Formally, the intrinsic gradient is the operator $\nabla: L^2(\ccalM)\rightarrow L^2(T\ccalM)$ mapping scalar functions $f \in L^2(\ccalM)$ to \textit{tangent vector functions} $\nabla f \in L^2(T\ccalM)$, where $L^2(T\ccalM)$ is the Hilbert space of vector fields over the manifold $\ccalM$. The tangent vector function $\nabla f(x) \in T_x\ccalM$ indicates the direction of the fastest change of manifold signal $f$ at point $x$. The adjoint of the intrinsic gradient is the \emph{intrinsic divergence}, defined as $\text{div}: L^2(T\ccalM)\rightarrow L^2(\ccalM)$. Interpreting the tangent vector field as the velocity field of a fluid, the intrinsic divergence can be seen as a measure of the net motion of the fluid \cite{bronstein2017geometric}. 

The Laplace-Beltrami (LB) operator of a manifold $\ccalM$ is defined as the operator $\ccalL: L^2(\ccalM) \to L^2(\ccalM)$ given by the function composition of the intrinsic divergence and the intrinsic gradient. When considered in the local coordinates supported on $T_x\ccalM$ \cite{canzani2013analysis}, the LB operator can be written as 
\begin{equation}\label{eqn:Laplacian}
    \ccalL f=-\text{div}\circ \nabla f=-\nabla \cdot \nabla f.
\end{equation}
Much like the Laplace operator in Euclidean domains (or the Laplacian matrix, in the case of graphs\cite{moon2012field}), the LB operator measures the total variation of function $f$, i.e., how much the value of $f$ at a point deviates from local average of the values of $f$ in its surroundings \cite{bronstein2017geometric}. Since the LB operator $\ccalL$, like the gradient $\nabla$, is a local operator depending on the tangent space $T_x\ccalM$ of each point $x \in \ccalM$, in the following we omit this dependence for the ease of presentation by writing $\ccalL=\ccalL_x$ and $\nabla=\nabla_x$.

The LB operator plays an important role in partial differential equations (PDEs), as it governs the dynamics of the diffusion of heat over manifolds as given by the \emph{heat equation}
\begin{equation}\label{eqn:heat}
    \frac{\partial u(x,t)}{\partial t}+\ccalL u(x,t)=0 \text{.}
\end{equation}
If $u\in L^2(\ccalM)$ measures the temperature over the manifold with $u(x,t)$ representing the temperature of point $x \in \ccalM$ at time $t \in \reals^{+}$, equation \eqref{eqn:heat} can be interpreted to mean that, at point $x$ on manifold $\ccalM$, the rate at which the manifold ``cools down'' is proportional to the difference between the temperature of $x$ and the local average of the temperature of the points in its neighborhood. With initial condition $u(x,0) = f(x)$ for all $x\in\ccalM$, the solution to this equation is given by
\begin{equation}\label{eqn:heat-solution}
    u(x,t) = e^{-t \ccalL}f(x) \text{,}
\end{equation}
which is the key support to implement the LB operator in the definitions proposed later.

The LB operator $\ccalL$ is self-adjoint and positive-semidefinite. Considering that $\ccalM$ is compact, the LB operator $\ccalL$ has a real positive eigenvalue spectrum $\{\lambda_i\}_{i=1}^\infty$ satisfying
\begin{equation}\label{eqn:laplacian-decomp}
\ccalL \bm\phi_i =\lambda_i \bm\phi_i
\end{equation}
where $\bm\phi_i$ is the eigenfunction associated with eigenvalue $\lambda_i$. The indices $i$ are such that the eigenvalues are ordered in increasing order as $0<\lambda_1\leq \lambda_2\leq \lambda_3\leq \hdots$ repeated according to their multiplicity. In particular, for a $d$-dimensional manifold, we have that $\lambda_i \propto i^{2/d}$ as a consequence of Weyl's law \cite{arendt2009weyl, musser2016weyl}. The eigenfunctions $\bm\phi_i$ are orthonormal and form a generalized eigenbasis of $L^2(\ccalM)$ in the intrinsic sense. Since $\ccalL$ is a total variation operator, the eigenvalues $\lambda_i$ can be interpreted as the canonical frequencies and the eigenfunctions $\bm\phi_i$ as the canonical oscillation modes of $\ccalM$. This further allows us to implement operator $\ccalL$ in the spectral domain.


\subsection{Manifold Filters}\label{sec_manifold_filters}
Time signals are processed by filters which compute the continuous time convolution of the input signal and the filter impulse response \cite{oppenheim1997signals};
images and high-dimensional Euclidean signals are processed by filters implementing multidimensional convolutions \cite{bishop2006pattern}; and graph signals are filtered by computing graph convolutions \cite{gama2020graphs}. In this paper, we define a manifold filter as the convolution of the filter impulse response $\tdh(t)$ and the manifold signal $f$. Note that the definition of the convolution operation, denoted as $\star_\ccalM$, leverages the heat diffusion dynamics described in \eqref{eqn:heat-solution}. 


\begin{definition}[Manifold filter]
\label{def:manifold-convolution}
Let $\tdh:\reals^+ \to \reals$ and let $f \in L^2(\ccalM)$ be a manifold signal. The manifold filter with impulse response $\tdh$, denoted as $\bbh$, is given by
\begin{align} \label{eqn:convolution-conti}
   g(x) = (\bbh f)(x) := (\tdh \star_\ccalM f) (x) := \int_0^\infty \tdh(t)u(x,t)\text{d}t ,
\end{align}
where $\tdh \star_\ccalM f$ is the \emph{manifold convolution} of $\tdh$ and $f$, and $u(x,t)$ is the solution of the heat equation \eqref{eqn:heat} with the initial condition $u(x,0)=f(x)$. 
\end{definition}

In a slight abuse of nomenclature, in the following we will use the terms manifold filter and manifold convolution interchangeably.

From Definition \ref{def:manifold-convolution}, we see that the manifold filter operates on manifold signals $f$ by (i) scaling the diffusion process \eqref{eqn:heat-solution} starting at $f$ by $\tdh$ and (ii) aggregating the outcome of the scaled diffusion process from $t=0$ to $t=\infty$. 

Substituting \eqref{eqn:heat} into the convolution definition in \eqref{eqn:convolution-conti} yields
\begin{equation} \label{eqn:manifold-conv-spatial}
   g(x) = (\bbh f)(x) =\int_0^\infty \tdh(t)e^{-t\ccalL}f(x)\text{d}t =  \bbh(\ccalL)f(x) \text{.}
\end{equation}
This alternative form uncovers the fact that the manifold convolution is a map from functions $f$ to functions $g$ that is \emph{parametric} on the Laplacian $\ccalL$ -- i.e., manifold convolutions are completely determined by the LB operator of the manifold. The exponential term $e^{-t\ccalL}$ can be seen as a diffusion or shift operation akin to a time delay in a linear time-invariant (LTI) filter \cite{oppenheim1997signals}, or as the graph shift operator in a linear shift-invariant (LSI) graph filter \cite{gama2020graphs}. Indeed, if we consider the manifold $\ccalM$ to be the real line, the manifold filter defined in \eqref{eqn:manifold-conv-spatial} recovers a LTI filter. If we consider it to be a set of points connected by a geometric graph, \eqref{eqn:manifold-conv-spatial} recovers an LSI graph filter. We discuss these special cases in detail in Appendix \ref{app:rem_convolution} and Section \ref{sec:discre_nn}.



\subsection{Frequency Representation of Manifold Filters}\label{sec_spectral_representation}

A manifold signal $f\in L^2(\ccalM)$ can be represented in the frequency domain of the manifold by projecting $f$ onto the LB operator eigenbasis \eqref{eqn:laplacian-decomp} as
\begin{equation}\label{eqn:f-decomp}
[\hat{f}]_i= \langle f, \bm\phi_i \rangle_{L^2(\ccalM)} = \int_\ccalM f(x) \bm\phi_i(x) \text{d} \mu(x) \text{,}
\end{equation} 
where we claim that $\hat{f}$ is the \textit{frequency representation} of the corresponding signal with $f=\sum_{i=1}^\infty [\hat{f}]_i \bm\phi_i$.

Frequency representations are useful because they help understand the frequency behavior of the manifold filter $\bbh(\ccalL)$. To see this, we consider the frequency representation of the manifold filter output $g$ in \eqref{eqn:manifold-conv-spatial}, which is
\begin{equation}
    [\hat{g}]_i = \int_\ccalM \int_0^\infty \tdh(t) e^{-t\ccalL} f(x) \text{d} t \bm\phi_i(x) \text{d} \mu(x) \text{.}
\end{equation}
Rearranging the integrals and substituting $e^{-t\ccalL}\phi_i = e^{-t\lambda_i}\phi_i$, we can get
\begin{equation}\label{eqn:projection}
    [\hat{g}]_i = \int_0^\infty \tdh(t) e^{-t\lambda_i} \text{d}t  [\hat{f}]_i \text{.}
\end{equation}
The expression relating $\hat{g}$ and $\hat{f}$ is called the \emph{frequency response} of the filter $\bbh(\ccalL)$. 

\begin{definition}[Frequency response]
\label{def:frequency-response}
The frequency response of manifold filter $\bbh(\ccalL)$ is given by
\begin{equation}\label{eqn:operator-frequency}
\hat{h}(\lambda)=\int_0^\infty \tdh(t) e^{- t \lambda  }\text{d}t \text{,   }\lambda \in (0, \infty).
\end{equation}
\end{definition}

An important consequence of Definition \ref{def:frequency-response} is that, since $\hat{h}(\lambda)$ is parametric on $\lambda$, the manifold filter is pointwise in the frequency domain. This can be seen by plugging \eqref{eqn:operator-frequency} into \eqref{eqn:projection}, and is stated explicitly in Proposition \ref{prop:filter-spectral}.


%

\begin{proposition}
\label{prop:filter-spectral}
Manifold filter $\bbh(\ccalL)$ is pointwise in the frequency domain, which is written as
\begin{equation}\label{eqn:convolution-general}
[\hat{g}]_i = \hat{h}(\lambda_i)[\hat{f}]_i \text{,   } i\in \mathbb{N}^+
\end{equation}
\end{proposition}

Definition \ref{def:frequency-response} also emphasizes that the frequency response of a manifold filter is independent of the underlying manifold. Note that, in \eqref{eqn:operator-frequency}, $\hat{h}(\lambda)$ is a function of an arbitrary scalar variable $\lambda$. To obtain the frequency behavior of this filter on a given manifold $\ccalM$, we need to evaluate $\hat{h}$ at the corresponding LB operator eigenvalues $\lambda_i$ [cf. \eqref{eqn:laplacian-decomp}]. If the manifold changes (or if we want to deploy the same filter on a different manifold $\ccalM'$), it suffices to reevaluate $\hat{h}$ at $\lambda_i'$, i.e., at the eigenvalues of the new LB operator $\ccalL'$.


\section{Stability of Manifold Filters with respect to Manifold Deformations}\label{sec_manifold_stability}



On the manifold $\ccalM$, we define a deformation as function $\tau: \ccalM \to \ccalM$ and the curvature distance between $x$ and the displaced $\tau(x)$ $\text{dist}(x,\tau(x))$ is upper bounded, i.e., $\tau(x)$ is a displaced point in the neighborhood of $x$, which holds for all $x\in\ccalM$. The deformation $\tau$ induces a corresponding tangent map $\tau_{*,x}: T_x\ccalM\rightarrow T_{\tau(x)}\ccalM$ {which is a linear map between the tangent spaces \cite{loring2011introduction}.} With the coordinate description of tangent map, the tangent map $\tau_{*,x}$ can be exactly represented by the Jacobian matrix $J_x(\tau)$. When $\text{dist}(x,\tau(x))$ is bounded, the Frobenius norm of $J_x(\tau)-I$ can also be upper bounded, and these bounds are used to measure the size of the deformation $\tau$. 

Let $f: \ccalM \to \reals$ be a manifold signal. Because $\ccalM$ is the codomain of $\tau$, $g = f \circ \tau$ maps points $\tau(x) \in \ccalM$ to $f(\tau(x)) \in \reals$, so that
the effect of a manifold deformation on the signal $f$ is a signal perturbation leading to a new signal $g$ supported on the same manifold. To understand the effect of this deformation on the LB operator, let $p = \ccalL g$. Since $p$ is also a signal on $\ccalM$, we may define an operator $\ccalL'$ mapping $f$ directly into $p$,
\begin{align}
\label{eqn:deform}
   p(x) = \ccalL' f(x) = \ccalL g(x) = \ccalL f(\tau(x)).
\end{align}
The operator $\ccalL'$ is the perturbed LB operator, which is effectively the new LB operator resulting from the deformation $\tau$. Assuming that the gradient field is smooth, the difference between $\ccalL'$ and $\ccalL$ is given by the following theorem. The proof is deferred to Appendix \ref{app:perturb}.


\begin{theorem} \label{thm:perturb}
{Let $\ccalL$ be the LB operator of manifold $\ccalM$.
Let $\tau:\ccalM\rightarrow \ccalM$ be a manifold perturbation such that $\text{dist}(x,\tau(x))\leq \epsilon$ and $J_x(\tau)= I+\Delta_x$ with $\|\Delta_x\|_F\leq \epsilon$ for all $x\in \ccalM$. If the gradient field is smooth, it holds that}
\begin{equation}
    \label{eqn:perturb-operator}
    \ccalL-\ccalL' = \bbE \ccalL + \ccalA ,
\end{equation}
where $\bbE$ and $\ccalA$ satisfy $\|\bbE\|=O(\epsilon)$ and $\|\ccalA\|_{op}= O(\epsilon)$. 
\end{theorem}

Therefore, the perturbation of the LB operator incurred by a manifold deformation $\tau$ is a combination of an absolute perturbation $\ccalA$ [cf. Definition \ref{def:abso-perturb}] and a relative perturbation $\bbE\ccalL$ [cf. Definition \ref{def:rela-perturb}]. This largely simplifies our analysis of stability. Since manifold filters are parametric on $\ccalL$ [cf. Proposition \ref{prop:filter-spectral}], it is sufficient to characterize their stability to deformations of the manifold by analyzing their behavior in the presence of absolute and relative LB perturbations. This is what we do in Sections \ref{subsec:filter-absolute} and \ref{subsec:filter-relative}.


\subsection{Stability of Manifold Filters to Absolute Perturbations}
\label{subsec:filter-absolute}
We start by analyzing the stability of manifold filters to absolute perturbations of the LB operator, which are introduced in Definition \ref{def:abso-perturb}.

\begin{definition}[Absolute perturbations] \label{def:abso-perturb}
Let $\ccalL$ be the LB operator of manifold $\ccalM$. An absolute perturbation of $\ccalL$ is defined as
\begin{equation}\label{eqn:abso-perturb}
\ccalL'-\ccalL=\ccalA,
\end{equation}
where the absolute perturbation operator $\ccalA$ is self-adjoint.
\end{definition}

Like $\ccalL$, the operator $\ccalL'$ resulting from the absolute perturbation of $\ccalL$ is self-adjoint due to the symmetry of $\ccalA$. Hence, it admits an eigendecomposition similar to \eqref{eqn:Laplacian}. When the filter coefficients are fixed, the frequency response of the manifold filter \eqref{eqn:convolution-general} can be obtained by evaluating $\hhath(\lambda)$ at each $\lambda_i$. Thus, to understand the effect of the perturbation on the filter we need to look at how the perturbation of the LB operator changes the eigenvalues $\lambda_i$. The challenge in this case is that the spectrum of $\ccalL$ is infinite-dimensional, i.e., there is an infinite (though countable) number of eigenvalue perturbations that need to be taken into account, which leads to an untractable infinite summation. As demonstrated by Proposition \ref{prop:finite_num}, however, large eigenvalues of $\ccalL$ tend to accumulate in certain parts of the real line. {This suggests a strategy to partition the spectrum into finite number of partitions, by treating the eigenvalues in the same partition similarly, we can turn to analyze the stability under the perturbations on these finite number of partitions with the discriminability of certain frequency components sacrificed.} 


\begin{proposition} \label{prop:finite_num}
Consider a $d$-dimensional compact manifold $\ccalM\subset \reals^\mathsf{N}$ and let $\ccalL$ be its LB operator with eigenvalues $\{\lambda_k\}_{k=1}^\infty$. Let $C_0$ be an arbitrary constant, $K_0(C_0)$ some finite constant depends on $C_0$ and $C_d$ the volume of the $d$-dimensional unit ball. Let $\text{Vol}(\ccalM)$ denote the volume of manifold $\ccalM$. For any $\alpha > 0$ and $d>2$, there exists $N_1$,
\begin{equation}
    N_1=\Big\lceil \left(\frac{\alpha d}{C_04\pi^2}\right)^{d/(2-d)}(C_d \text{Vol}(\ccalM))^{2/(2-d)} \Big\rceil
\end{equation}
such that, for all $k>\max\{N_1, K_0(C_0)\}$, 
$$\lambda_{k+1}-\lambda_k\leq \alpha$$
\end{proposition}
\begin{proof}
This is a direct consequence of Weyl's law \cite[Chapter~1]{arendt2009mathematical}\cite{musser2016weyl}, See Appendix \ref{sup:weyl}.
\end{proof}

Given this asymptotic behavior, we can divide the eigenvalues into a finite number of partitions by placing eigenvalues that are less than $\alpha>0$ apart from each other in groups. This spectrum separation strategy is described in Definition \ref{def:alpha-spectrum}. To achieve it, we will need a specific type of manifold filter called Frequency Difference Threshold (FDT) filter as introduced in Definition \ref{def:alpha-filter}.


\begin{definition}[$\alpha$-separated spectrum]\label{def:alpha-spectrum}
An $\alpha$-separated spectrum of a LB operator $\ccalL$ is defined as a partition $\Lambda_1(\alpha) \cup \ldots\cup \Lambda_N(\alpha)$ such that all $\lambda_i \in \Lambda_k(\alpha)$ and $\lambda_j \in \Lambda_l(\alpha)$, $k \neq l$, satisfy
\begin{align}\label{eqn:alpha-spectrum}
|\lambda_i - \lambda_j| > \alpha \text{.}
\end{align}
\end{definition}


\begin{definition}[$\alpha$-FDT filter]\label{def:alpha-filter}
An $\alpha$-frequency difference threshold ($\alpha$-FDT) filter is defined as a filter $\bbh(\ccalL)$ whose frequency response satisfies
\begin{equation} \label{eq:fdt-filter}
    |\hhath(\lambda_i)-\hhath(\lambda_j)|\leq \delta_k \mbox{ for all } \lambda_i, \lambda_j \in \Lambda_k(\alpha) 
\end{equation}
with $\delta_k\leq \delta$ for $k=1, \ldots,N$.
\end{definition}

\begin{figure}
\centering

\pgfplotsset{xtick style={draw=none}}
\pgfplotsset{ytick style={draw=none}}

\def \thisplotscale {3}
\def \unit {\thisplotscale cm}

\def \frequencyresponse 
     { 0.3*exp(-(0.8*(x-2))^2) 
       -0.2 *exp(-(1*(x-4.5))^2) 
       + 0.2*exp(-(1*(x-10))^2) 
       + 0.55}
       
\newcommand{\drawgaussian}[7]
{
   \def \filter { ( \height*exp(-(1/\width*(x-\center))^2) ) }
   \def \center {#1}
   \def \width  {#2}
   \def \height {#3}

   \def \xminhere   {#4}
   \def \xmaxhere   {#5}

   \addplot[ domain     = \center-3.0*\width:\center + 3.0*\width, 
             samples    = 80, 
             color      = #6,
             line width = #7]
           { \filter };

   \addplot[ domain     = \xminhere:\center-3.0*\width, 
             samples    = 2, 
             color      = #6,
             line width = #7]
           { 0 };
   
   \addplot[ domain     = \center+3.0*\width:\xmaxhere, 
             samples    = 2, 
             color      = #6,
             line width = #7]
           { 0 };
}

\begin{tikzpicture}[x = 1*\unit, y=1*\unit]
\def \factorx {2.4/8}
\def \deltax  {0.5*\factorx}
\def \shadeshift  {0.05}

\begin{axis}[scale only axis,
             width  = 2.5*\unit,
             height = 0.8*\unit,
             xmin = 0, xmax=8,
             xtick = { 0,  1.05 ,  1.89, 3.2 , 4.1, 5.1},
             xticklabels = {\black{\footnotesize $0$},
             				\blue{\footnotesize $\ \Lambda_1$}, 
                            \blue{\footnotesize $\ \Lambda_2$}, 
                            \blue{\footnotesize $\ \Lambda_3$}, 
                            \blue{\footnotesize $\ \Lambda_4$}, 
                            \blue{\footnotesize $\ \Lambda_5$},
                            },
             ymin = -0, ymax = 1.15,
             ytick = {0.55},
             yticklabels = {\black{\footnotesize $h(\lambda)$}},
             typeset ticklabels with strut,
             enlarge x limits=false]
             
\addplot [fill=black, fill opacity=0.1, draw opacity = 0]
       coordinates {
            (1, 0) (1.25, 0) (1.25, 1.15) (1, 1.15)  };       
            
\addplot [fill=black, fill opacity=0.1, draw opacity = 0]
       coordinates {
            (1.8, 0) (2, 0) (2, 1.15) (1.8, 1.15)  };   
            
            
\addplot [fill=black, fill opacity=0.1, draw opacity = 0]
      coordinates {
            (5, 0) (5.2, 0) (5.2, 1.15) (5, 1.15)  };
            
\addplot [fill=black, fill opacity=0.1, draw opacity = 0]
      coordinates {
            (6.72, 0) (8, 0) (8, 1.15) (6.72, 1.15)  }; 
            

\addplot[ domain=0:8, 
          samples = 80, 
          color = black,
          line width = 1.2]
         {\frequencyresponse};

\addplot[samples at = {0, 1, 1.25,   1.8, 2,   3.2,  4.1,   5, 5.2,     6.72,  6.9, 7.1, 7.3,  7.54, 7.7, 7.83},
         color = blue!60, 
          ycomb, 
          mark=otimes*, 
          mark options={blue!60}]
         {\frequencyresponse};.
         
\drawgaussian{1.125}{0.4}{0.9}{0}{8}{red!60}{1.2}
 \addplot+[ samples at   = {0, 1, 1.25,   1.8, 2,   3.2,  4.1,   5, 5.2,     6.72,  6.9, 7.1, 7.3,  7.54, 7.7, 7.83}, 
              solid, 
              ycomb, 
              mark         = otimes*, 
              mark size    = 1.5pt,
              line width   = 0.8,              
              color        =  blue!60, 
              mark options = {red!60}
            ]
            { \filter };

\drawgaussian{1.9}{0.3}{0.9}{0}{8}{red!60}{1.2}
 \addplot+[ samples at   = {0, 1, 1.25,   1.8, 2,   3.2,  4.1,   5, 5.2,     6.72,  6.9, 7.1, 7.3,  7.54, 7.7, 7.83}, 
              solid, 
              ycomb, 
              mark         = otimes*, 
              mark size    = 1.5pt,
              line width   = 0.8,              
              color        =  blue!60, 
              mark options = {red!60}
            ]
            { \filter };
\drawgaussian{3}{0.4}{0.9}{0}{8}{red!60}{1.2}
 \addplot+[ samples at   = {0, 1, 1.25,   1.8, 2,   3.2,  4.1,   5, 5.2,     6.72,  6.9, 7.1, 7.3,  7.54, 7.7, 7.83}, 
              solid, 
              ycomb, 
              mark         = otimes*, 
              mark size    = 1.5pt,
              line width   = 0.8,              
              color        =  blue!60, 
              mark options = {red!60}
            ]
            { \filter };
\drawgaussian{4}{0.4}{0.9}{0}{8}{red!60}{1.2}
 \addplot+[ samples at   = {0, 1, 1.25,   1.8, 2,   3.2,  4.1,   5, 5.2,     6.72,  6.9, 7.1, 7.3,  7.54, 7.7, 7.83}, 
              solid, 
              ycomb, 
              mark         = otimes*, 
              mark size    = 1.5pt,
              line width   = 0.8,              
              color        =  blue!60, 
              mark options = {red!60}
            ]
            { \filter };
\drawgaussian{5.1}{0.4}{0.9}{0}{8}{red!60}{1.2}
 \addplot+[ samples at   = {0, 1, 1.25,   1.8, 2,   3.2,  4.1,   5, 5.2,     6.72,  6.9, 7.1, 7.3,  7.54, 7.7, 7.83}, 
              solid, 
              ycomb, 
              mark         = otimes*, 
              mark size    = 1.5pt,
              line width   = 0.8,              
              color        =  blue!60, 
              mark options = {red!60}
            ]
            { \filter };
\drawgaussian{7.36}{1}{0.9}{0}{8}{red!60}{1.2}
 \addplot+[ samples at   = {0, 1, 1.25,   1.8, 2,   3.2,  4.1,   5, 5.2,     6.72,  6.9, 7.1, 7.3,  7.54, 7.7, 7.83}, 
              solid, 
              ycomb, 
              mark         = otimes*, 
              mark size    = 1.5pt,
              line width   = 0.8,              
              color        =  blue!60, 
              mark options = {red!60}
            ]
            { \filter };

\end{axis}
\end{tikzpicture}    
  \caption{Illustration of an $\alpha$-FDT filter. The $x$-axis stands for the spectrum with each sample representing an eigenvalue. The gray shaded areas show the grouping of the eigenvalues according to Definition \ref{def:alpha-spectrum}. The red lines show a set of $\alpha$-FDT filters that can discriminate each eigenvalue group.}
\label{fig:alpha}
\end{figure}

In an $\alpha$-separated spectrum, eigenvalues $\lambda_i \in \Lambda_k(\alpha)$ and $\lambda_j \in \Lambda_l(\alpha)$ in different sets ($k \neq l$) are at least $\alpha$ away from each other. Conversely, eigenvalues $\lambda_i, \lambda_j \in \Lambda_k(\alpha)$ are no more than $\alpha$ apart. This partitioning creates several eigenvalue groups spaced by at least $\alpha$. Note that the sets $\Lambda_k(\alpha)$ can have any size and, in particular, they can be singletons.

The partitioning of the spectrum described in Definition \ref{def:alpha-spectrum} is achieved by an $\alpha$-FDT filter. This filter separates the spectrum of the manifold by {assigning similar frequency responses---that deviate no more than $\delta_k$ from each other---to eigenvalues $\lambda_i \in \Lambda_k(\alpha)$, $1 \leq k \leq N$. In other words, an $\alpha$-FDT filter does not discriminate between eigenvalues $\lambda_i, \lambda_j \in\Lambda_k(\alpha)$. Importantly, the $\delta_k$ in Definition \ref{def:alpha-spectrum} are finite, so that they can be bounded by some $\delta$.} 

To obtain manifold filters that are stable to absolute perturbations of $\ccalL$, we also need these filters to be Lipschitz continuous as shown in Definition \ref{def:lipschitz}.

\begin{definition}[Lipschitz filter] \label{def:lipschitz}
A filter is $A_h$-Lispchitz if its frequency response is Lipschitz continuous with Lipschitz constant $A_h$, i.e,
\begin{equation}
    |\hhath(a)-\hhath(b)| \leq A_h |a-b|\text{ for all } a,b \in (0,\infty)\text{.}
\end{equation}
\end{definition}

Between the eigenvalue groups, the filters that we consider are assumed to be $A_h$-Lipschitz continuous. This means that, in regions of the spectrum where the $\Lambda_k(\alpha)$ are singletons, the filter can vary with slope at most $A_h$ as shown in Figure \ref{fig:alpha}. Note that we can always construct convolutional filters \eqref{eqn:operator-frequency} that are both Lipschitz continuous and $\alpha$-FDT.

Under mild assumptions on the amplitude of the frequency response $\hat{h}$ (Assumption \ref{ass:filter_function}), it can be shown that Lipschitz continuous $\alpha$-FDT filters are stable to absolute perturbations of the LB operator. This result is stated in Theorem \ref{thm:stability_abs_filter}.

\begin{assumption}[Non-amplifying filters] \label{ass:filter_function}
The filter frequency response $\hhath:\reals^+\rightarrow\reals$ is non-amplifying. I.e., for all $\lambda\in(0,\infty)$, $\hhath$ satisfies $|\hhath(\lambda)|\leq 1$.
\end{assumption}

Note that this assumption is rather reasonable, because the filter frequency response $\hhath(\lambda)$ can always be normalized.

\begin{theorem}[Manifold filter stability to absolute perturbations]\label{thm:stability_abs_filter}
Consider a manifold $\ccalM$ with LB operator $\ccalL$. Let $\bbh(\ccalL)$ be an $\alpha$-FDT manifold filter [cf. Definition \ref{def:alpha-filter}] and $A_h$-Lipschitz [cf. Definition \ref{def:lipschitz}]. 
Consider an absolute perturbation $\ccalL'=\ccalL + \ccalA$ of the LB operator $\ccalL$ [cf. Definition \ref{def:abso-perturb}] where $\|\ccalA\| \leq  \epsilon < \alpha$. Then, under Assumption \ref{ass:filter_function} it holds that
 \begin{align}\label{eqn:stability_nn}
    \nonumber\|\bbh(\ccalL)f-& \bbh(\ccalL')f\|_{L^2(\ccalM)} \leq \\& \left(\frac{\pi N_s \epsilon}{\alpha-\epsilon}+A_h\epsilon+ 2(N-N_s)\delta\right) \|f\|_{L^2(\ccalM)},
 \end{align}
{where $N$ is the size of the $\alpha$-separated spectrum partition [cf. Definition \ref{def:alpha-spectrum}] and $N_s$ is the number of singletons.}
\end{theorem}
\begin{proof}
See Appendix \ref{app:stability_abs_filter}.
\end{proof}

Provided that $\epsilon \ll \alpha$, FDT filters are thus stable to absolute perturbations of the LB operator $\ccalL$. The stability bound
depends on (i) the continuity of the FDT filter as measured by the Lipschitz constant $A_h$ and (ii) its frequency difference threshold $\alpha$, which affects the bound directly as well as indirectly through the number of partitions $N$. Note that this bound consists of three terms. The first corresponds to
the difference between the eigenfunctions of $\ccalL$ and $\ccalL'$, which affects the stability bound by changing projection directions. The second stems from the distance between the original and perturbed eigenvalues. Finally, the third reflects the bounded fluctuation of the filter frequency response within the same eigenvalue group.

The bound in Theorem \ref{thm:stability_abs_filter} can be simplified by setting $\delta=\pi\epsilon/(2\alpha-2\epsilon)$ as in Corollary \ref{cor:stability_abs_filter}.

\begin{corollary}\label{cor:stability_abs_filter}
Setting $\delta=\pi\epsilon/(2\alpha-2\epsilon)$, under the same assumptions of Theorem \ref{thm:stability_abs_filter} it holds that
\begin{align}\label{eqn:stability_abs_filter_cor}
    \|\bbh(\ccalL)f-\bbh(\ccalL')f\|_{L^2(\ccalM)} \leq  \left(\frac{\pi N}{\alpha-\epsilon}+A_h\right)\epsilon \|f\|_{L^2(\ccalM)}.
 \end{align}
\end{corollary}

A particular case of Theorem \ref{thm:stability_abs_filter}, the simplified stability bound in Corollary \ref{cor:stability_abs_filter} is helpful to understand the effect of the filter spectrum on stability as well as of the size of the perturbation. 
In particular, from Corollary \ref{cor:stability_abs_filter} we can tell that the filter is more stable if the Lipschitz constant $A_h$ is small and the frequency difference threshold $\alpha$ is large.
On the other hand, small $A_h$ and large $\alpha$ mean that the filter is less discriminative. This reveals a stability-discriminability trade-off where discriminability should be understood as the ability to tell frequencies apart. In other words, we propose $\alpha$-FDT filter to maintain the stability by trying to discriminate only eigenvalue groups instead of every single eigenvalue. Importantly, this trade-off is not related to the magnitude of the frequencies that the filters amplify (as is the case in, e.g., \cite{gama2020stability}, \cite{ruiz2020graphon}).We will keep observing this trade-off throughout the stability analysis of both manifold filters and MNNs. More details about this characteristic will be discussed in Section \ref{subsec:discussion}.

\subsection{Stability of Manifold Filters to Relative Perturbations}
\label{subsec:filter-relative}

Relative perturbations of the LB operator are defined similarly as follows.

\begin{definition}[Relative perturbations] \label{def:rela-perturb}
Let $\ccalL$ be the LB operator of manifold $\ccalM$. A relative perturbation of $\ccalL$ is defined as
\begin{equation}\label{eqn:rela-perturb}
\ccalL'-\ccalL=\bbE \ccalL,
\end{equation}
where the relative perturbation term $\bbE \ccalL$ is self-adjoint.
\end{definition}

Like absolute perturbations, relative perturbations also perturb the eigenvalues and eigenfunctions of $\ccalL$. However, in the case of relative perturbations, the perturbations to the eigenvalues are proportional to their absolute values {[cf. Lemma \ref{lem:eigenvalue_relative} in Appendix \ref{app:lemmas}]}. Relative perturbations thus require a different spectrum separation strategy to guarantee stability. With relative perturbations, larger eigenvalues are impacted with larger perturbation values, which can send eigenvalues originally $\alpha$-close to each other to different groups as Figure \ref{fig:gamma} shows. Therefore, we will need a different type of filter implementing a different type of spectrum separation. Our strategy is inspired by Proposition \ref{prop:finite_num_rela}, which is another variation of Weyl's law. 

\begin{proposition} \label{prop:finite_num_rela}
Let $\ccalM$ be a $d$-dimensional compact embedded manifold in $\reals^\mathsf{N}$ with LB operator $\ccalL$, and let $\{\lambda_k\}_{k=1}^\infty$ denote the eigenvalues of $\ccalL$. Let $C_0$ denote an arbitrary constant and $K_0(C_0)$ some finite constant depends on $C_0$. For any $\gamma > 0$, there exists $N_2$ given by
\begin{equation}
    N_2=\lceil (((\gamma+1)/C_0)^{d/2}-1)^{-1} \rceil
\end{equation}
such that, for all $k>\max\{N_2, K_0(C_0)\}$, it holds that $$\lambda_{k+1}-\lambda_k\leq \gamma\lambda_k.$$
\end{proposition}
\begin{proof}
This is a direct consequence of Weyl's law \cite[Chapter~1]{arendt2009mathematical}\cite{musser2016weyl}. See Appendix \ref{sup:weyl}.
\end{proof}

Hence, to enforce stability we need to separate the spectrum relatively to the ratio between neighboring eigenvalues. This partitioning is called $\gamma$-separated spectrum and formalized in Definition \ref{def:frt-spectrum}. A $\gamma$-separated spectrum is achieved by a so-called Frequency Ratio Threshold (FRT) filters. We introduce them in Definition \ref{def:frt-filter}.


\begin{definition}[$\gamma$-separated spectrum.]\label{def:frt-spectrum}
A $\gamma$-separated spectrum of a LB operator $\ccalL$ is defined as a partition $\Lambda_1(\gamma)\cup\hdots \cup\Lambda_M(\gamma)$ such that all $\lambda_i\in\Lambda_k(\gamma)$ and $\lambda_j\in\Lambda_l(\gamma)$, $k\neq l$, satisfy
\begin{equation}
\label{eqn:frt-spectrum}
\left|\frac{\lambda_i}{\lambda_j}-1 \right|>\gamma.
\end{equation}
\end{definition}

\begin{definition}[$\gamma$-FRT filter.]\label{def:frt-filter}
A $\gamma$-frequency ratio threshold ($\gamma$-FRT) filter is defined as a filter $\bbh(\ccalL)$ whose frequency response satisfies
\begin{equation}\label{eqn:frt-filter}
       |\hhath(\lambda_i)-\hhath(\lambda_j)|\leq \delta_k,\text{ for all } \lambda_i,\lambda_j\in\Lambda_k(\gamma)
\end{equation}
with $\delta_k\leq \delta$ for $k=1,2\hdots,M$.
\end{definition}

In a $\gamma$-separated spectrum, the sets $\Lambda_k(\gamma)$ are built based on eigenvalue distances \textit{relative} to the eigenvalues' magnitudes and weighted by the parameter $\gamma$. Eigenvalues $\lambda_j \in \Lambda_k(\gamma)$ and $\lambda_i \in \Lambda_l(\gamma)$ in different groups (i.e., $k \neq l$) are at least $\gamma\min(\lambda_i,\lambda_j)$ apart from each other. This means that, for $\lambda_i, \lambda_{i+1} \in \Lambda_k(\gamma)$, $\lambda_{i+1} - \lambda_i \leq \gamma \lambda_i$.

A $\gamma$-FRT filter achieves a spectrum separation in Definition \ref{def:frt-spectrum} by giving eigenvalues $\lambda_i, \lambda_j \in \Lambda_k(\gamma)$ very similar frequency responses differing by at most plus or minus $\delta_k \leq \delta$. Meanwhile, eigenvalues belonging to different sets $\Lambda_k(\gamma)$ and $\Lambda_l(\gamma)$, $k \neq l$, are treated independently, and their frequency responses can vary a lot.

To make a manifold filter stable to relative perturbations of the LB operator, we need a further restriction on their continuity. Lipschitz continuity [cf. Definition \ref{def:lipschitz}] is not enough because for a Lipschitz filter the difference in frequency response for a perturbed eigenvalue grows with the eigenvalue magnitude, since the eigenvalue perturbation is relative. Therefore, we need our filters to be \textit{integral Lipschitz} continuous as is described in Definition \ref{def:int-lipschitz}.

\begin{definition}[Integral Lipschitz filter] \label{def:int-lipschitz}
A filter is integral Lipschitz continuous with constant $B_h$ if its frequency response satisfies
\begin{equation}\label{eqn:filter_function}
    |\hhath(a)-\hhath(b)|\leq \frac{B_h |a-b| }{(a+b)/2} \text{ for all } a,b \in (0,\infty)\text{.}
\end{equation}
\end{definition}

\begin{figure}
\centering

\pgfplotsset{xtick style={draw=none}}
\pgfplotsset{ytick style={draw=none}}

\def \thisplotscale {3}
\def \unit {\thisplotscale cm}

\def \frequencyresponse 
     { 0.3*exp(-(0.8*(x-2))^2) 
       - 0.2*exp(-(1*(x-4.5))^2) 
       + 0.2*exp(-(1*(x-10))^2) 
       + 0.55}

\newcommand{\drawgaussian}[7]
{
   \def \filter { ( \height*exp(-(1/\width*(x-\center))^2) ) }
   \def \center {#1}
   \def \width  {#2}
   \def \height {#3}

   \def \xminhere   {#4}
   \def \xmaxhere   {#5}

   \addplot[ domain     = \center-3.0*\width:\center + 3.0*\width, 
             samples    = 80, 
             color      = #6,
             line width = #7]
           { \filter };

   \addplot[ domain     = \xminhere:\center-3.0*\width, 
             samples    = 2, 
             color      = #6,
             line width = #7]
           { 0 };
   
   \addplot[ domain     = \center+3.0*\width:\xmaxhere, 
             samples    = 2, 
             color      = #6,
             line width = #7]
           { 0 };
}

\begin{tikzpicture}[x = 1*\unit, y=1*\unit]
\def \factorx {2.4/8}
\def \deltax  {0.5*\factorx}
\def \shadeshift  {0.05}

\begin{axis}[scale only axis,
             width  = 2.5*\unit,
             height = 0.8*\unit,
             xmin = 0, xmax=8,
             xtick = { 0,  0.85 ,  1.3,1.9 , 3.2 , 4.1, 5.1},
             xticklabels = {\black{\footnotesize $0$},
             				\blue{\footnotesize $\ \Lambda_1$}, 
                            \blue{\footnotesize $\ \Lambda_2$}, 
                            \blue{\footnotesize $\ \Lambda_3$}, 
                            \blue{\footnotesize $\ \Lambda_4$},
                            \blue{\footnotesize $\ \Lambda_5$}, 
                            \blue{\footnotesize $\ \Lambda_6$},
                            },
             ymin = -0, ymax = 1.15,
             ytick = {0.55},
             yticklabels = {\black{\footnotesize $h(\lambda)$}},
             typeset ticklabels with strut,
             enlarge x limits=false]
             
            
\addplot [fill=black, fill opacity=0.1, draw opacity = 0]
       coordinates {
            (1.8, 0) (2, 0) (2, 1.15) (1.8, 1.15)  };   

\addplot [fill=black, fill opacity=0.1, draw opacity = 0]
      coordinates {
            (5, 0) (5.2, 0) (5.2, 1.15) (5, 1.15)  };

\addplot [fill=black, fill opacity=0.1, draw opacity = 0]
      coordinates {
            (6.72, 0) (8, 0) (8, 1.15) (6.72, 1.15)  };

\addplot[ domain=0:8, 
          samples = 80, 
          color = black,
          line width = 1.2]
         {\frequencyresponse};

\addplot[samples at = {0, 1, 1.25, 1.8, 2,  3.2, 4.1, 5, 5.2,  6.72,  6.9, 7.1, 7.3,  7.54, 7.7, 7.83},
         color = blue!60, 
          ycomb, 
          mark=otimes*, 
          mark options={blue!60}]
         {\frequencyresponse};

\drawgaussian{1}{0.1}{0.9}{0}{8}{red!60}{1.2}
 \addplot+[ samples at   = {0, 1, 1.25,   1.8, 2,   3.2,  4.1,   5, 5.2,     6.72,  6.9, 7.1, 7.3,  7.54, 7.7, 7.83}, 
              solid, 
              ycomb, 
              mark         = otimes*, 
              mark size    = 1.5pt,
              line width   = 0.8,              
              color        =  blue!60, 
              mark options = {red!60}
            ]
            { \filter };        
         
\drawgaussian{1.25}{0.1}{0.9}{0}{8}{red!60}{1.2}
 \addplot+[ samples at   = {0, 1, 1.25,   1.8, 2,   3.2,  4.1,   5, 5.2,     6.72,  6.9, 7.1, 7.3,  7.54, 7.7, 7.83}, 
              solid, 
              ycomb, 
              mark         = otimes*, 
              mark size    = 1.5pt,
              line width   = 0.8,              
              color        =  blue!60, 
              mark options = {red!60}
            ]
            { \filter };

\drawgaussian{1.9}{0.3}{0.9}{0}{8}{red!60}{1.2}
 \addplot+[ samples at   = {0, 1, 1.25,   1.8, 2,   3.2,  4.1,   5, 5.2,     6.72,  6.9, 7.1, 7.3,  7.54, 7.7, 7.83}, 
              solid, 
              ycomb, 
              mark         = otimes*, 
              mark size    = 1.5pt,
              line width   = 0.8,              
              color        =  blue!60, 
              mark options = {red!60}
            ]
            { \filter };
\drawgaussian{3}{0.4}{0.9}{0}{8}{red!60}{1.2}
 \addplot+[ samples at   = {0, 1, 1.25,   1.8, 2,   3.2,  4.1,   5, 5.2,     6.72,  6.9, 7.1, 7.3,  7.54, 7.7, 7.83}, 
              solid, 
              ycomb, 
              mark         = otimes*, 
              mark size    = 1.5pt,
              line width   = 0.8,              
              color        =  blue!60, 
              mark options = {red!60}
            ]
            { \filter };
\drawgaussian{4}{0.4}{0.9}{0}{8}{red!60}{1.2}
 \addplot+[ samples at   = {0, 1, 1.25,   1.8, 2,   3.2,  4.1,   5, 5.2,     6.72,  6.9, 7.1, 7.3,  7.54, 7.7, 7.83}, 
              solid, 
              ycomb, 
              mark         = otimes*, 
              mark size    = 1.5pt,
              line width   = 0.8,              
              color        =  blue!60, 
              mark options = {red!60}
            ]
            { \filter };
\drawgaussian{5.1}{0.4}{0.9}{0}{8}{red!60}{1.2}
 \addplot+[ samples at   = {0, 1, 1.25,   1.8, 2,   3.2,  4.1,   5, 5.2,     6.72,  6.9, 7.1, 7.3,  7.54, 7.7, 7.83}, 
              solid, 
              ycomb, 
              mark         = otimes*, 
              mark size    = 1.5pt,
              line width   = 0.8,              
              color        =  blue!60, 
              mark options = {red!60}
            ]
            { \filter };
\drawgaussian{7.36}{1}{0.9}{0}{8}{red!60}{1.2}
 \addplot+[ samples at   = {0, 1, 1.25,   1.8, 2,   3.2,  4.1,   5, 5.2,     6.72,  6.9, 7.1, 7.3,  7.54, 7.7, 7.83}, 
              solid, 
              ycomb, 
              mark         = otimes*, 
              mark size    = 1.5pt,
              line width   = 0.8,              
              color        =  blue!60, 
              mark options = {red!60}
            ]
            { \filter };
\end{axis}
\end{tikzpicture}

  \caption{Illustration of a $\gamma$-FRT filter. The $x$-axis stands for the spectrum with each sample representing an eigenvalue. The gray shaded area shows the grouping of the eigenvalues according to Definition \ref{def:frt-spectrum}. The red lines show a set of $\alpha$-FDT filters that can discriminate each eigenvalue group. }
\label{fig:gamma}
\end{figure}

Integral Lipschitz filters can be seen as Lipschitz filters with variable Lipschitz constant, which decreases with $\lambda$. E.g., on the interval $(a,b)$, the filter in Definition \ref{def:int-lipschitz} behaves as a Lipschitz filter with Lipschitz constant $2B_h/(a+b)$. When $a$ and $b$ are close, this condition can be approximated by $|a\hat{h}'(a)|\leq B_h$ for all $a\in(0,\infty)$. This implies that the filter function flattens for high-frequency eigenvalues as shown in Figure \ref{fig:gamma}.

Under Assumption \ref{ass:filter_function}, integral Lipschitz $\gamma$-FRT filters are stable to relative perturbations as stated in Theorem \ref{thm:stability_rela_filter}.

\begin{theorem}[Manifold filter stability to relative perturbations]\label{thm:stability_rela_filter}
Consider a manifold $\ccalM$ with LB operator $\ccalL$. Let $\bbh(\ccalL)$ be a $\gamma$-FRT filter with $\delta=\pi\epsilon/(2\gamma-2\epsilon+2\gamma\epsilon)$ [cf. Definition \ref{def:frt-filter}] and $B_h$-integral Lipschitz [cf. Definition \ref{def:int-lipschitz}]. 
Consider a relative perturbation $\ccalL'=\ccalL + \bbE\ccalL$ of the LB operator $\ccalL$ [cf. Definition \ref{def:rela-perturb}] where $\|\bbE\| \leq \epsilon < \gamma$. 
Then, under Assumption \ref{ass:filter_function} it holds that
 \begin{align}\label{eqn:stability_rela_filter}
 \|\bbh(\ccalL)f-&\bbh(\ccalL')f\|_{L^2(\ccalM)}   \leq  \left(\frac{\pi M\epsilon}{\gamma-\epsilon+\gamma\epsilon}+ \frac{2{B_h}\epsilon}{2-\epsilon} \right) \|f\|_{L^2(\ccalM)} 
 \end{align}
where $M$ is the size of the $\gamma$-separated spectrum partition [cf. Definition \ref{def:frt-spectrum}].
\end{theorem}
\begin{proof}
See Appendix \ref{app:stability_rela_filter}.
\end{proof}

When $\epsilon$ is sufficiently small ($\epsilon \ll \min(\gamma,2)$, which is typically the case with deformations such as the one in Theorem \ref{thm:perturb}), the denominators on the right hand side of \eqref{eqn:stability_rela_filter} are approximately equal to $\gamma$ and $2$ respectively. Hence, $\gamma$-FRT integral Lipschitz filters are stable to relative perturbations of the LB operator.
Besides appearing in the bound in Theorem \ref{thm:stability_rela_filter}), the frequency ratio threshold $\gamma$ also affects stability indirectly through the partition size $M$. With a larger $\gamma$, fewer eigenvalues will be in singleton sets, thus decreasing $M$ and improving stability. A smaller integral Lipschitz constant $B_h$ also increases stability. However, small $B_h$ and large $\gamma$ make for smoother filters which in turn lead to a less discriminative manifold filter. Therefore, integral Lipschitz $\gamma$-FRT filters also exhibit a trade-off between discriminality and stability. 

\begin{remark}
\normalfont
\label{rem:filters}
By comparing the illustrations of $\alpha$-FDT filter (Definition \ref{def:alpha-filter}) and $\gamma$-FRT filter (Definition \ref{def:frt-filter}) in Figure \ref{fig:alpha} and Figure \ref{fig:gamma}, we see that in practice these filters have a similar frequency behavior because, due to Weyl's law [cf. Proposition \ref{prop:finite_num} and \ref{prop:finite_num_rela}], high frequency components will eventually be grouped in the same group and thus share similar frequency responses. Therefore, the main difference between these filters is their effects on the low-frequency components. In the low frequency spectrum, the eigengaps $\lambda_{i+1}-\lambda_i$ may be smaller than the difference threshold $\alpha$, but larger than the relative ratio threshold $\gamma \lambda_i$ due to $\lambda_i$ being small. However, for appropriate values of $\gamma$ a filter may be both FDT and FRT. This will be shown in Section \ref{sec:stability_nn}.
\end{remark}

\section{Stability of Manifold Neural Networks}
\label{sec:stability_nn}

Manifold neural networks (MNNs) are deep convolutional architectures comprised of $L$ layers, where each layer consists of two components: a convolutional filter bank and a pointwise nonlinearity. 
At each layer $l=1,2,\hdots, L$, the convolutional filters map the incoming $F_{l-1}$ features from layer $l-1$ into $F_l$ intermediate linear features given by
\begin{equation}\label{eqn:mnn}
y_l^p(x) =  \sum_{q=1}^{F_{l-1}} \bbh_l^{pq}(\ccalL) f_{l-1}^q(x),
\end{equation}
where $\bbh_l^{pq}(\ccalL)$ is the filter mapping the $q$-th feature from layer $l-1$ to the $p$-th feature of layer $l$ as in \eqref{eqn:convolution-general}, for $1\leq q\leq F_{l-1}$ and $1\leq p\leq F_{l}$. The intermediate features are then processed by a pointwise nonlinearity $\sigma: \reals\rightarrow \reals$ as
\begin{equation}\label{eqn:mnn}
f_l^p(x) = \sigma\left(y_l^p(x) \right).
\end{equation}
The nonlinearity $\sigma$ processes each feature individually and we further make an assumption on its continuity as follows. 

\begin{assumption}[Normalized Lipschitz activation functions]\label{ass:activation}
 The activation function $\sigma$ is normalized Lipschitz continous, i.e., $|\sigma(a)-\sigma(b)|\leq |a-b|$, with $\sigma(0)=0$.
\end{assumption}
 
Note that this assumption is rather reasonable, since most common activation functions (e.g., the ReLU, the modulus and the sigmoid) are normalized Lipschitz by design.

At the first layer of the MNN, the input features are the input data $f^q$ for $1\leq q\leq F_0$. At the output of the MNN, the output features are given by the outputs of the $L$-th layer, i.e., $f_L^p$ for $1 \leq p \leq F_L$. To represent the MNN more succinctly, we may gather the impulse responses of the manifold convolutional filters $\bbh_l^{pq}$ across all layers in a function set $\bbH$, and define the MNN map $\bbPhi(\bbH,\ccalL, f)$. This map emphasizes that the MNN is parameterized by both the filter functions and the LB operator $\ccalL$. We next will analyze the stability of $\bbPhi(\bbH,\ccalL,f)$ with respect to perturbations on the underlying manifold.




\subsection{Stability of MNNs to LB Operator Perturbations}
\label{subsec:stability_mnn_abs}

MNNs inherit stability to perturbations of the LB operator from the manifold filters that compose the filterbanks in each one of their layers. This result is stated in general form---encompassing both absolute and relative perturbations---in the following theorem.

\begin{theorem}[MNN stability]\label{thm:stability_nn}
Consider a compact embedded manifold $\ccalM$ with LB operator $\ccalL$. Let $\bm\Phi(\bbH,\ccalL,f)$ be an $L$-layer MNN on $\ccalM$ \eqref{eqn:mnn} with $F_0=F_L=1$ input and output features and $F_l=F,l=1,2,\hdots,L-1$ features per layer. 
The filters $\bbh(\ccalL)$ and nonlinearity functions satisfy  Assumptions \ref{ass:filter_function} and \ref{ass:activation} respectively. Let $\ccalL'$ be the perturbed LB operator [cf. Definition \ref{def:abso-perturb} or Definition \ref{def:rela-perturb}] with $\max\{\alpha, 2, |\gamma/1-\gamma|\}\gg \epsilon$. If the manifold filters satisfy $\| \bbh(\ccalL)f -\bbh(\ccalL')f \|_{L^2(\ccalM)}\leq C_{per} \epsilon \|f\|_{L^2(\ccalM)}$, it holds that
\begin{align}\label{eqn:stability_nn}
   \nonumber \|\bm\Phi(\bbH,\ccalL,f)-\bm\Phi(\bbH,\ccalL',f)\|_{L^2(\ccalM)} \leq LF^{L-1}C_{per}\epsilon \|f\|_{L^2(\ccalM)}.
 \end{align}
\end{theorem}
\begin{proof}
See Appendix \ref{app:stability_nn}.
\end{proof}

Theorem \ref{thm:stability_nn} reflects that the stability of the MNN is affected by the hyperparameters of the MNN architecture and the stability constant of the manifold filters $C_{per}$. More explicitly, the stability bound grows linearly with the number of layers $L$ and exponentially with the number of features $F$ where the rate is determined by $L$. This stability result also shows that there is a linear dependence on the stability constant $C_{per}$ of manifold filters $\bbh(\ccalL)$ and the perturbation size $\epsilon$. As we have shown in Section \ref{subsec:filter-absolute} and \ref{subsec:filter-relative}, the stability constant is determined by the form of the perturbations (Definition \ref{def:abso-perturb} or Definition \ref{def:rela-perturb}) as well as the spectrum separation achieved by the specific manifold filters (Definition \ref{def:alpha-filter} or Definition \ref{def:frt-filter}) with corresponding Lipschitz conditions (Definition \ref{def:lipschitz} or Definition \ref{def:int-lipschitz}). We address the specific cases as follows.

\begin{proposition}
\label{prop:stability-nn}
With the same conditions as Theorem \ref{thm:stability_nn}, consider the following perturbation models.
\begin{enumerate}
    \item  If the perturbed LB operator $\ccalL'$ is an absolute perturbation, i.e., $\ccalL'=\ccalL+\bbA$ [cf. Definition \ref{def:abso-perturb}] with $\|\bbA\|\leq \epsilon$ and the manifold filters $\bbh(\ccalL)$ are $\alpha$-FDT [cf. Definition \ref{def:alpha-filter}] with $\alpha\gg\epsilon$ and $A_h$-Lipschitz continuous [Definition \ref{def:lipschitz}] with $\delta = \pi\epsilon/(2\alpha)$, we have
    \begin{equation}
        C_{per} = \frac{\pi N}{\alpha} + A_h,
    \end{equation}
    where $N$ is the size of the $\alpha$-separated spectrum partition [cf. Definition \ref{def:alpha-spectrum}].
    \item If the perturbed LB operator $\ccalL'$ is a relative perturbation, i.e. $\ccalL'=\ccalL+\bbE\ccalL$ [cf. Definition \ref{def:rela-perturb}] with $\|\bbE\|\leq \epsilon$, and the manifold filters $\bbh(\ccalL)$ are $\gamma$-FRT [cf. Definition \ref{def:frt-filter}] with $\gamma/(1-\gamma)\gg\epsilon$ and $B_h$-integral Lipschitz continuous [Definition \ref{def:int-lipschitz}] with $\delta = \pi\epsilon/(2\gamma)$, we have
    \begin{equation}
        C_{per} = \frac{\pi M}{\gamma} + B_h,
    \end{equation}
    where $M$ is the size of the $\gamma$-separated spectrum partition [cf. Definition \ref{def:frt-spectrum}].
\end{enumerate}

\end{proposition}
\begin{proof}
The conclusions follow directly from Theorem \ref{thm:stability_nn} combined with Theorem \ref{thm:stability_abs_filter} or Theorem \ref{thm:stability_rela_filter} under the corresponding assumptions.
\end{proof}

Combining Theorem \ref{thm:stability_nn} with Proposition \ref{prop:stability-nn}, we observe that $\alpha$-FDT manifold filters with Lipschitz continuity can be composed to construct MNNs which are stable to absolute perturbations; while $\gamma$-FRT manifold filters with integral Lipschitz continuity can be composed to construct MNNs which are stable relative perturbations of the LB operator.
Explicitly, by inserting the stability constant $C_{per}$ in \eqref{eqn:stability_nn}, we see that other than the perturbation size $\epsilon$, there are three terms that determine the stability of MNNs. The first term is $LF^{L-1}$, which, as we have already discussed, is decided by the number of layers and filters in the MNN architecture. This term arises due to the propagation of the underlying operator perturbations across all the manifold filters in all layers of the MNN. The second term is $\pi N/\alpha$ or $\pi M/\gamma$, which results from the deviations of the eigenfunctions as well as from the frequency response variations within the same eigenvalue partition. Finally, the third term, $A_h$ or $B_h$, is given by the Lipschitz or integral Lipschitz constants which are decided during the filter design or the training process. It is important to note that the stability constant $C_{per}$ brings along the trade-off between stability and discriminability. However, unlike manifold filters, MNNs can be both stable and discriminative. This arises from the effects of nonlinear activation functions, as we discuss in further detail in Section \ref{subsec:discussion}.

\subsection{Stability of MNNs to Manifold Deformations} \label{subsec:stability_mnn_def}

In Theorem \ref{thm:stability_nn} and Proposition \ref{prop:stability-nn}, we established the conditions under which MNNs are stable to either absolute or relative perturbations of the LB operator as defined in Definitions \ref{def:abso-perturb} and \ref{def:rela-perturb}. 
Since a manifold deformation $\tau:\ccalM\rightarrow \ccalM$, with $\text{dist}(x,\tau(x))\leq \epsilon$ and $\|J_x(\tau)-I\|_F\leq \epsilon$ for all $x\in\ccalM$, translates into both an absolute and a relative perturbation of the Laplace-Beltrami operator, {MNNs composed of manifold filters meeting all of these conditions in items 1 and 2 of Proposition \ref{prop:stability-nn},
i.e., the manifold filters are $\alpha$-FDT and $\gamma$-FRT, and both Lipschitz continuous and integral Lipschitz continuous, can be proved to be stable under the manifold deformation.} The spectrum can be made to be both $\alpha$-separated and $\gamma$-separated by making sure the eigenvalues in different partitions satisfy both \eqref{eqn:alpha-spectrum} and \eqref{eqn:frt-spectrum}. Assuming that all of these conditions are met, we can state our main result---that MNNs are stable to deformations of the manifold---as follows.


\begin{theorem}
Let $\ccalM$ be a compact embedded manifold with LB operator $\ccalL$ and $f$ be a manifold signal. We construct $\bm\Phi(\bbH,\ccalL,f)$ as a MNN on $\ccalM$ \eqref{eqn:mnn} where the filters $\bbh(\ccalL)$ are $\alpha$-FDT [cf. Definition \ref{def:alpha-filter}], $\alpha/\lambda_1$-FRT [cf. Definition \ref{def:frt-filter}], $A_h$-Lipschitz [cf. Definition \ref{def:lipschitz}] and $B_h$-integral Lipschitz [cf. Definition \ref{def:int-lipschitz}]. Consider a deformation on $\ccalM$ as $\tau :\ccalM\rightarrow \ccalM$ where $\text{dist}(x,\tau(x))\leq \epsilon$ and $J_x(\tau)=I+\Delta_x$ with $\|\Delta_x\|_F\leq \epsilon$ for all $x\in\ccalM$ and $\epsilon\ll \min(\alpha/\lambda_1,\alpha,2)$. Under Assumptions \ref{ass:filter_function} and \ref{ass:activation} it holds that
\begin{align}
    \|\bm\Phi(\bbH,\ccalL,f)-\bm\Phi(\bbH,\ccalL',f)\|_{L^2(\ccalM)}=O(\epsilon)\|f\|_{L^2(\ccalM)}.
\end{align}
\end{theorem}

Together, Theorem \ref{thm:perturb} and Theorem \ref{thm:stability_nn} imply that MNNs are stable to the manifold deformations $\upsilon$ introduced in the beginning of this section. This is because these deformations spawn a perturbation of the LB operator that consists of both an absolute and a relative perturbation. For stability to hold, the filters that make up the layers of the MNN need to be $\alpha$-FDT [cf. Definition \ref{def:alpha-filter}], $\gamma$-FRT [cf. Definition \ref{def:frt-filter}], Lipschitz [cf. Definition \ref{def:lipschitz}] and integral Lipschitz [cf. Definition \ref{def:int-lipschitz}]. We can propose an easier special case to relate $\alpha$ and $\gamma$ by utilizing the spectrum property of LB operator. By setting the $\alpha$-FDT filter with $\alpha = \gamma\lambda_1$, eigenvalues $\lambda_i,\lambda_{i+1}\in\Lambda_k(\alpha)$ would lead to $\lambda_i,\lambda_{i+1}\in\Lambda_l(\gamma)$ due to the fact that
\begin{align}
    \lambda_{i+1}-\lambda_i \leq \alpha=\gamma\lambda_1\leq \gamma \lambda_i,
\end{align}
with $\lambda_1$ indexed as the smallest eigenvalue in the spectrum. The requirement that the filter be $\alpha$-FDT can be removed as long as $\lambda_1>0$ and $\alpha=\gamma \lambda_1$, since a $\gamma$-FRT filter is always $\gamma\lambda_1$-FDT, i.e. $\alpha$-FDT.

\section{Discussions}
\label{subsec:discussion}

{\myparagraph{Stability vs. discriminability tradeoff}}
In both stability theorems for manifold filters (Theorems \ref{thm:stability_abs_filter}, \ref{thm:stability_rela_filter}) and in the stability theorem for MNNs (Theorem \ref{thm:stability_nn}), the stability  bounds depend on the frequency partition threshold ($\alpha$ or $\gamma$), the number of total partitions ($N$ or $M$) and the Lipschitz continuity constant ($A_h$ or $B_h$). 
The frequency partition threshold and the number of partitions have a combined effect on stability. As indicated by Definitions \ref{def:alpha-spectrum} and \ref{def:frt-spectrum}, a larger frequency threshold leads to a smaller number of singletons, as eigenvalues that would otherwise be separated for small thresholds end up being grouped when the threshold is large. While a large frequency threshold results in a larger number of partitions that contain more than one eigenvalue, the total number of partitions ($N$ or $M$) either stays the same or decreases because the number of eigenvalues does not exceed the number of partitions [cf. Proposition \ref{prop:finite_num} or \ref{prop:finite_num_rela}]. Thus, a larger frequency threshold and a smaller number of partitions both lead to a smaller stability bound. Simultaneously, a large frequency threshold makes the spectrum separated more sparsely. Therefore, a large number of eigenvalues are amplified in a similar manner, which makes the filter function less discriminative. {Considering in the limit, if the frequency threshold goes to infinity, all the eigenvalues tend to be grouped and the filter would not discriminate the whole spectrum. This would lead to a very stable filter but there is no discriminability at all.} The Lipschitz constant ($A_h$ or $B_h$) affects stability and discriminability in similar ways. Smaller Lipschitz constants decrease the stability bound, but lead to smoother filter functions giving similar frequency responses to different eigenvalues.
Hence, in both manifold filters and MNNs we observe a trade-off between stability and discriminability. Nevertheless, in MNNs this trade-off is alleviated due to the presence of nonlinearities as discussed below.

 

\myparagraph{Pointwise nonlinearity} 
As demonstrated by Propositions \ref{prop:finite_num} and \ref{prop:finite_num_rela}, large eigenvalues of LB operator tend to be grouped together in one large group and share similar frequency responses. This is part of the reason why manifold filters have a stability-discriminability tradeoff, which implies that they cannot be stable and discriminative at the same time. However, in MNNs this problem is circumvented with the addition of nonlinearities. Nonlinearities have the effect of scattering the spectral components all over the eigenvalue spectrum. In the MNN, they mix the frequency components by spilling spectral components associated with the large eigenvalues that tend to be very close onto the smaller eigenvalues that could be more separated, where they can then be discriminated by the manifold filters in the following layer. This is consistent with the role of nonlinear activation functions in graph neural networks (GNNs) \cite{gama2020stability}, which can be see as instantiations of MNNs on discrete samples of the manifold as further discussed in Section \ref{sec:discre_nn}.

\myparagraph{Comparison with graphons} The graphon is another infinite-dimensional model that can represent the limit of convergent sequences of graphs, and a series of works have proved stability of graphon neural networks and the transferability of GNNs sampled from them \cite{ruiz2020graphon, ruiz2021transferability, ruiz2021graph, maskey2021transferability, keriven2020convergence}. {Manifolds can represent the limits of relatively sparse graphs including $\epsilon$-graphs and $k$-NN graphs \cite{calder2019improved}.    }
While graphons can also be seen as the limit model of relatively sparse graphs \cite{keriven2020convergence}, embedded manifolds in high-dimensional spaces are more realistic geometric models with physical interpretations in a number of application scenarios, such as point clouds, 3D shape segmentation and classification. Other important differences are that (i) the stability analysis on graphon models in \cite{ruiz2021graphon,ruiz2020graphon} focuses on deformations to the adjacency matrix of the graph, which can be translated directly as perturbations of the graphon operator, and that (ii) in the case of graphons, only an absolute perturbation model makes sense since given that the graphon spectrum is bounded a relative perturbation can always be bounded by an absolute perturbation. Meanwhile, deformations to the manifold domain translate into a combination of absolute and relative perturbations of the LB operator, and the fact that the LB operator spectrum is unbounded makes the effects of absolute and relative perturbations distinct, especially in the high-frequency domain.


\section{From Manifold Neural Networks to Graph Neural Networks}
\label{sec:discre_nn}

MNNs are built from manifold convolutional filters (Definition \ref{def:manifold-convolution}) operating on a continuous manifold and over an infinite time horizon. 
This makes it impractical to implement directly the architecture described by \eqref{eqn:mnn} in applications. In this section, we discuss how MNNs are implemented in practice over a set of discrete samples from the manifold in a finite and discrete time frame.

\subsection{{Discretization in the Space Domain}}


In practice, the explicit form of the manifold and of its LB operator are unknown. What we typically have access to is a point cloud representation of the manifold, i.e., a discrete set of sampling points.
From these points' coordinates, the structure of the manifold is approximated by a geometric or a nearest neighbor graph \cite{dunson2021spectral,belkin2008towards,calder2019improved}. The LB operator is then approximated by the graph Laplacian,
which can be shown to converge to the LB operator as the number of sampling points grows \cite{dunson2021spectral} \cite{calder2019improved}.

Explicitly, suppose that $X= \{x_1, x_2,\dots, x_n\}$ is a set of $n$ points sampled i.i.d. from measure $\mu$ of manifold $\ccalM$, which is embedded in $\reals^\mathsf{N}$. We can construct a complete weighted symmetric graph $\bbG_n$ by taking the sampled points to be the vertices of the graph and setting the edge weights based on the Euclidean distance between pairs of points. Specifically, the weight $w_{ij}$ associated with edge $(i,j)$ is given by
\begin{equation}\label{eqn:weight}
    w_{ij}=\frac{1}{n}\frac{1}{t_n(4\pi t_n)^{k/2}}\exp\left(-\frac{\|x_i-x_j\|^2}{4t_n}\right),
\end{equation}
where $\|x_i-x_j\|$ is the Euclidean distance between points $x_i$ and $x_j$ while $t_n$ is a parameter associated with the chosen Gaussian kernel \cite{belkin2008towards}. The adjacency matrix $\bbA_n \in \reals^{n\times n}$ is thus defined as $[\bbA_n]_{ij}=w_{ij}$ for $1 \leq i,j\leq n$ and the corresponding graph Laplacian matrix $\bbL_n$ \cite{merris1995survey} is given by
\begin{equation}
    \bbL_n = \mbox{diag}(\bbA_n \boldsymbol{1})-\bbA_n.
\end{equation}
We interpret $\bbL_n$ the Laplacian operator of the constructed graph $\bbG_n$. 
Similarly, we define a uniform sampling operator $\bbP_n: L^2(\ccalM)\rightarrow L^2(\bbG_n)$ to sample manifold signals. Given a manifold signal $f$, we can use operator $\bbP_n$ to sample graph signals $\bbx_n \in \reals^n$ as
\begin{equation}
\label{eqn:sampling}
    \bbx_n = \bbP_n f\text{ with }[\bbx_n]_i = f(x_i), \quad  x_i \in X,
\end{equation}
where the $i$-th entry of the graph signal $\bbx_n$ is the manifold signal $f$ evaluated at the sample point $x_i$. 

In Section \ref{sec_manifold_filters}, we have shown that the manifold filter $\bbh$ is parametric with the LB operator $\ccalL$. Therefore, we can also parameterize $\bbh$ with the discrete graph Laplacian operator $\bbL_n$, which is written as,
\begin{equation}\label{eqn:sample_manifold_convolution}
    \bbz_n = \int_{0}^\infty \tdh(t) e^{-t\bbL_n} \text{d}t \bbx_n=\bbh(\bbL_n)\bbx_n, \; \bbx_n, \bbz_n \in\reals^n, 
\end{equation}
where $\bbz_n$, the output of the filter, is now a discrete graph signal.
By cascading these discrete filters operated on graph $\bbG_n$ and pointwise nonlinearities layer after layer, we can then approximate the MNN on $\bbG_n$ as
\begin{equation}\label{eqn:dis-mnn}
    \bbx_l^p = \sigma\left(\sum_{q=1}^{F_{l-1}} \bbh_l^{pq}(\bbL_n) \bbx^q_{l-1} \right),
\end{equation}
where $\bbh_l^{pq}(\bbL_n)$ maps the $q$-th feature in the $l-1$-th layer to the $p$-th feature in the $l$-th layer, $1\leq q \leq F_{l-1}$ and $1\leq p\leq F_l$, and $F_l$ denotes the number of features in the $l$-th layer (we have dropped the subscript $n$ in $\bbx_l^p$ and $\bbx_{l-1}^q$ for simplicity).
After gathering the filter functions in the set $\bbH$, this neural network 
can be represented more succinctly as $\bm\Phi(\bbH, \bbL_n, \bbx)$. 

Equation \eqref{eqn:dis-mnn} is a consistent approximation of the MNN because, as $n$ goes to infinity, the discrete graph Laplacian operator $\bbL_n$ of the graph $\bbG_n$ converges to the LB operator $\ccalL$ of the manifold $\ccalM$, and the sampled graph signal $\bbx_n$ converges to the manifold signal $f$ \cite{belkin2008towards}. These facts combinely imply that the output of the neural network on the graph $\bbG_n$ converges to the output of the neural network on the continuous manifold as stated in the following.
\begin{proposition}
\label{prop:convergence} 
Let $X=\{x_1, x_2,...x_n\}$ be $n$ points sampled i.i.d. from measure $\mu$ of $d$-dimensional manifold $\ccalM \subset \reals^{\mathsf{N}}$, with corresponding sampling operator $\bbP_n$ \eqref{eqn:sampling}. Let $\bbG_n$ be a discrete graph approximation of $\ccalM$ constructed from $X$ as in \eqref{eqn:weight} with $t_n = n^{-1/(d+2+a)}$ and $a>0$. Let $\bm\Phi(\bbH, \cdot, \cdot)$ be a neural network parameterized either by the LB
operator $\ccalL$ of the manifold $\ccalM$ or the graph Laplacian operator $\bbL_n$ of $\bbG_n$ with the filters in $\bbH$ satisfying Definition \ref{def:alpha-filter} and \ref{def:lipschitz}. It holds that
\begin{equation}
    \lim_{n\rightarrow \infty} \|\bm\Phi(\bbH,\bbL_n,\bbP_n f) - \bbP_n \bm\Phi(\bbH,\ccalL,f) \|_{L^2(\bbG_n)} = 0,
\end{equation}
with the limit taken in probability.
\end{proposition}
\begin{proof}
See Appendix \ref{app:convergence}.
\end{proof}

This proposition provides theoretical support to state that neural networks constructed from the discrete Laplacian $\bbL_n$ converge to MNN and thus can inherit the stability properties of the MNN.


\subsection{Discretization in the Time Domain}

In order to learn an MNN \eqref{eqn:mnn}, we need to learn the manifold convolutional filters $\bbh_l^{pq}$. This means that we need to learn the impulse responses $\tilde{h}(t)$ in Definition \ref{def:manifold-convolution}. However, learning continuous functions $\tilde{h}$ is computationally infeasible,
so we sample $\tilde{h}$ over fixed intervals of duration $T_s$ and parameterize the filter with coefficients $h_k = \tilde{h}(k T_s)$, $k =0 ,1, 2\dots$. Setting the sampling interval to $T_s=1$ for simplicity, the discrete-time manifold convolution can be written as
\begin{equation}
\label{eqn:manifold_convolution_discrete}
    \bbh(\ccalL) f(x)= \sum_{k=0}^{\infty} h_k e^{-k\ccalL}f(x) 
\end{equation}
where $\{h_k\}_{k=0}^\infty$ are called the filter coefficients or taps.

Yet, learning \eqref{eqn:manifold_convolution_discrete} is still impractical because there is an uncountable number of parameters $h_k$. To address this, we fix a time horizon of $K$ time steps and rewrite \eqref{eqn:manifold_convolution_discrete} as
\begin{equation}
\label{eqn:manifold_convolution_discrete_finite}
    \bbh(\ccalL) f(x)= \sum_{k=0}^{K-1} h_k e^{-k\ccalL}f(x)
\end{equation}
which can be seen as a finite impulse response (FIR) filter with shift operator $e^{-\ccalL}$. 
Indeed, the frequency response of this filter [cf. Proposition \ref{prop:filter-spectral}] is given by
\begin{equation}
    \hat{h}(\lambda)= \sum_{k=0}^{K-1} h_k e^{-k\lambda}.
\end{equation}

Combining \eqref{eqn:sample_manifold_convolution} and \eqref{eqn:manifold_convolution_discrete_finite}, we can bring the discretization over the spatial and time domains together to rewrite the convolution operation on the discretized manifold and in the discrete-time domain, explicitly,
\begin{equation}
\label{eqn:discrete_manifold_convolution_discrete}
  \bbz_n=  \bbh(\bbL_n) \bbx_n = \sum_{k=0}^{K-1} h_k e^{-k\bbL_n}\bbx_n.
\end{equation}
Equation \eqref{eqn:discrete_manifold_convolution_discrete} recovers the definition of the graph convolution \cite{gama2020graphs} with graph shift operator $e^{-\bbL_n}$. This means that in practice we implement MNNs as graph neural networks (GNNs). Therefore, the stability behavior of the GNN can be seen as a proxy for the stability behavior of the MNN. We will leverage this idea in the numerical experiments of Section \ref{sec:simu}.

\begin{remark}
\normalfont
\label{rem:timediscretization}
{We analyze the convergence of GNNs to MNNs when graphs are constructed based on uniformly sampled points from the manifold in Proposition \ref{prop:convergence}. The convergence holds when GNNs and MNNs share the same filter parameters while the filter functions are continuous in the time domain. We discuss the discretization of manifold filters in the time domain from a practical aspect that neural networks are usually trained and operated in digital systems with digital filters. }
\end{remark}

\section{Numerical Experiments}
\label{sec:simu}

\begin{figure*}[h]
\centering
\includegraphics[trim=130 0 80 0,clip,width=0.15\textwidth]{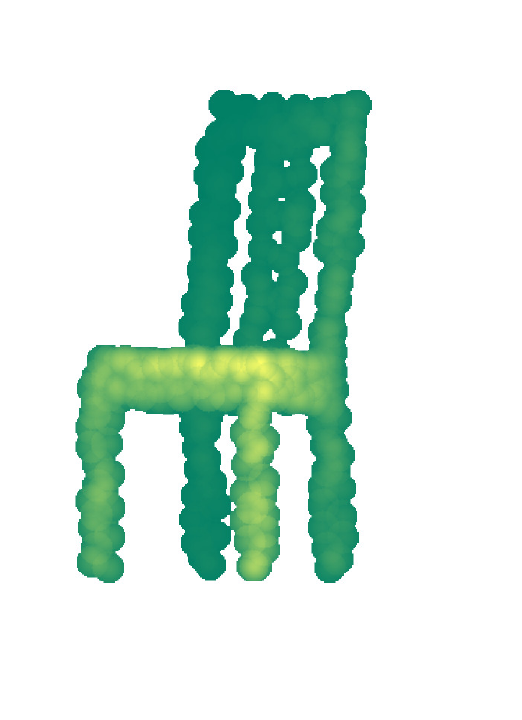}
\includegraphics[trim=90 0 100 0,width=0.15\textwidth]{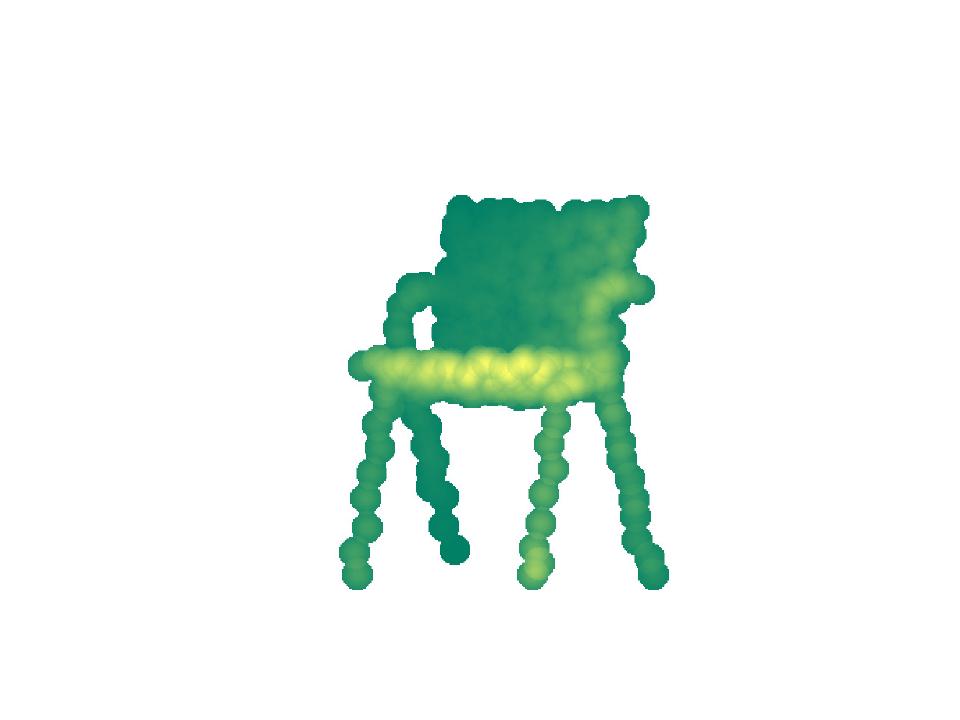}  
\includegraphics[trim=60 0 100 0,width=0.15\textwidth]{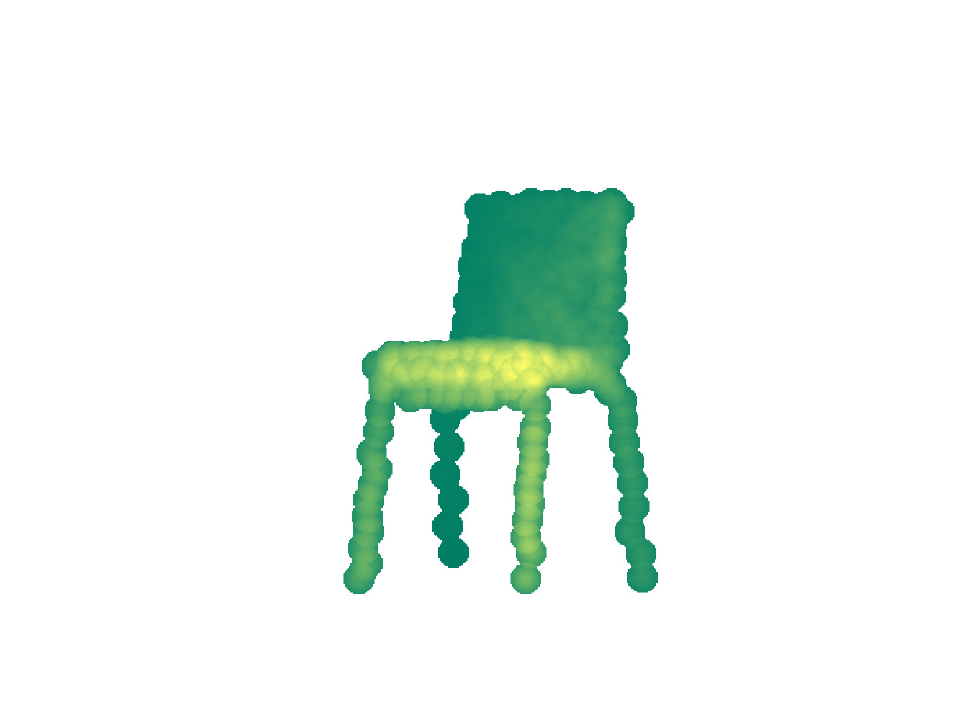} 
\includegraphics[trim=60 0 50 0,width=0.15\textwidth]{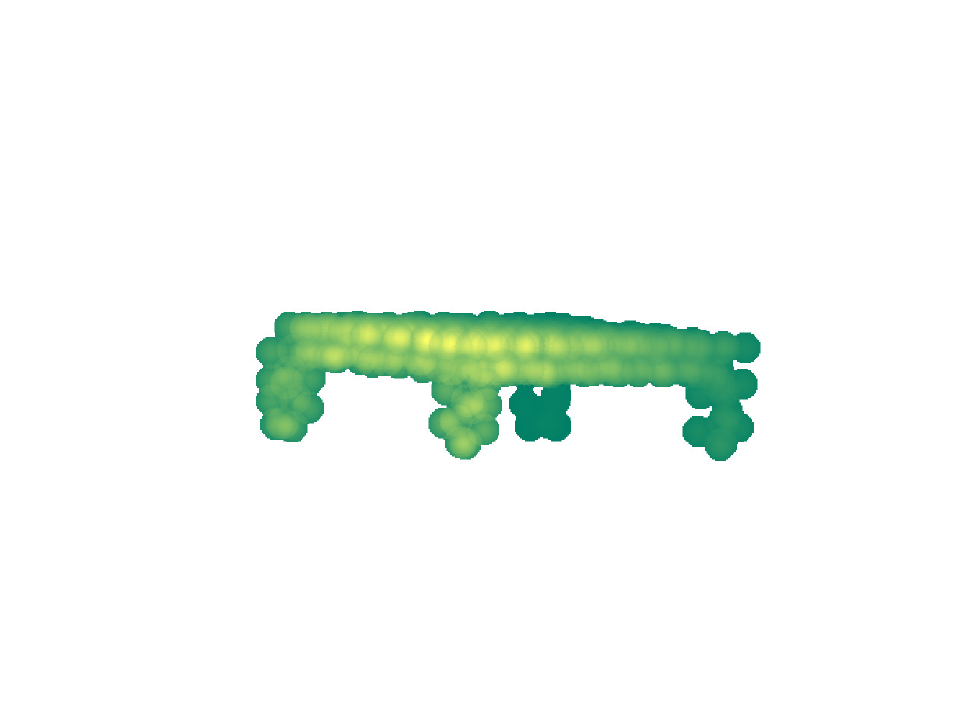}  
\includegraphics[trim=60 0 100 0,width=0.15\textwidth]{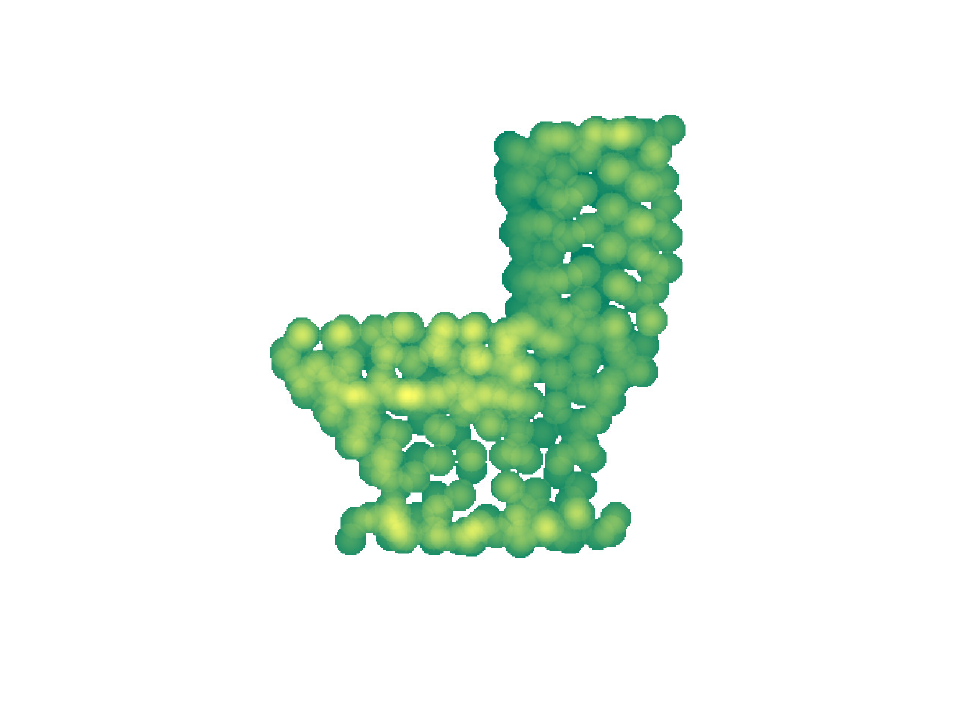}  
\includegraphics[trim=60 0 100 0,width=0.15\textwidth]{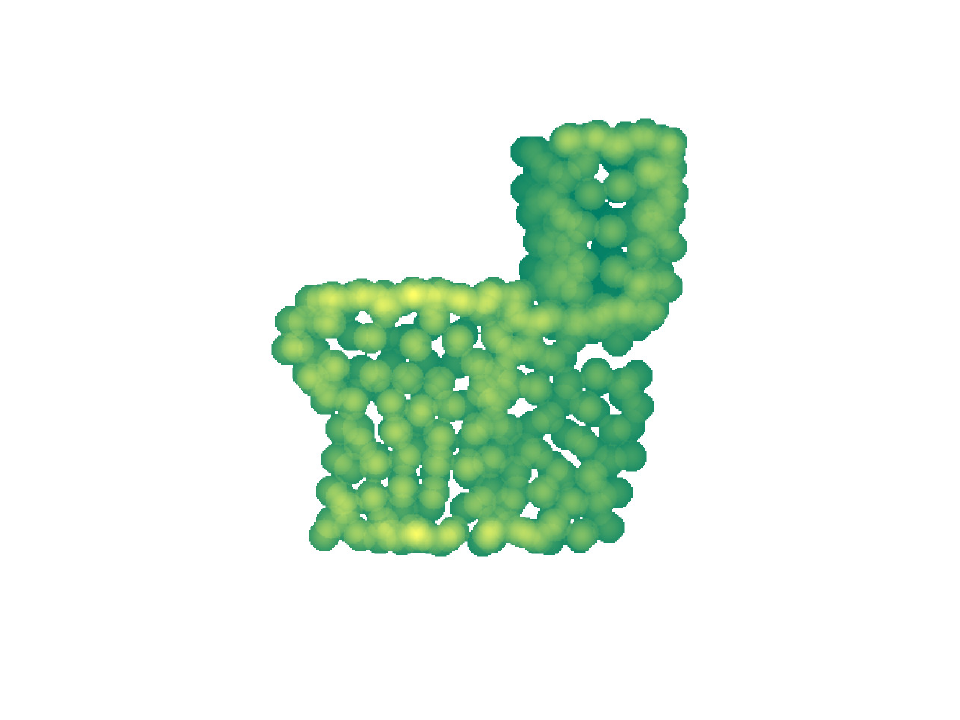} 

\caption{Point cloud models with 300 sampling points in each model. Our goal is to identify chair models from other models such as toilet and table. }
\label{fig:points}
\end{figure*}

\myparagraph{Dataset} We evaluate our MNN stability results on the ModelNet10 \cite{wu20153d} classification problem. The dataset contains 3991 meshed CAD models from 10 categories for training and 908 models for testing. For each model, 300 points are uniformly randomly sampled from all points of the model to form the point cloud. Each point is characterized by the 3D coordinates as features. We formulate the problem by modeling a dense graph neural network model to approximate MNN. Each node in the graph can be modeled as the sampling point and each edge weight is constructed based on the distance between each pair of nodes.  In this work our goal is to identify the CAD model for chairs as is illustrated in Figure \ref{fig:points} with the models for chair labeled as 1 and the others as 0. We deform the underlying manifold structure by adding random perturbations to the coordinates of the sampling points. By comparing the differences of the classification error rates, we aim to show that MNNs with Lipschitz continuous and integral Lipschitz continuous manifold filters are stable via looking into the performance of the approximated GNNs. 

\myparagraph{Neural network architectures} We build dense graphs to approximate the point cloud models. We use the coordinates of each point as node features. By connecting a point with all the other points in the point cloud, the edge weight is defined based the distance between every two points and a Gaussian kernel. The Laplacian matrix is calculated for each input point cloud model. We implement different architectures, including Graph Filters (GF) and Graph Neural Networks (GNN) with 1 and 2 layers,  to solve the classification problem. The architectures with a single layer contain $F_0=3$ input features which are the 3d coordinates of each point, $F_1=64$ output features and $K=5$ filter taps. While the architectures with 2 layers has another layer with $F_2= 32$ features and $5$ filter taps. We use the ReLU as nonlinearity. The learned graph filters are not regularized in architectures with `NoPel' while graph filters in the other architectures are both Lipschitz and integral Lipschitz. {We approximate the spectrum partitions with the continuous assumptions of the filter functions.} All architectures also include a linear readout layer mapping the final output features to a binary scalar that estimates the classification. 

\myparagraph{Discriminability experiment} We train all the architectures with an ADAM optimizer \cite{kingma2014adam} with learning rate set as 0.005 and decaying factors as 0.9, 0.999 by minimizing the entropy loss. The training point cloud models are divided in batches of 10 over 40 epochs. We run 5 random point samplings for all the architectures and we show the average classification error rates across these realizations as well as the standard deviation in Table \ref{tb:results}. We can observe that with the use of non-linearity, Graph Neural Networks perform better compared with Graph Filters. Architectures with more layers learn more accurate models which also leads to better performances.  

\begin{table}[h]
\centering
\begin{tabular}{l|c} \hline
Architecture    & error rates   \\ \hline
GNN1Ly	& $8.04 \% \pm 0.88\% $   \\ \hline
GNN2Ly		& $4.30\% \pm 2.64\%$   \\ \hline
GF1Ly		& $13.77\% \pm 6.87\%$   \\ \hline
GF2Ly	& $12.22\% \pm 7.89\%$   \\ \hline
\end{tabular}
\caption{Classification error rates for model `chair' in the test dataset. Average over 5 data realizations. The number of nodes is $n=300$.}
\label{tb:results}
\vspace{-3mm}
\end{table} 

\myparagraph{Stability experiment} We test the same trained Graph Neural Networks and Graph Filters with 2 layers on perturbed test point cloud models with different perturbation levels. We perturb the test point clouds by adding a Gaussian random variable with mean $\epsilon$ and variance $2\epsilon$ to each coordinate of every sampling point, which can be seen as a deformation of the underlying manifold. We measure the stability by computing the difference between the error rates achieved based on the original test point cloud models and the perturbed ones. In Figure \ref{fig:sim}, we see that this difference increases when the perturbations become larger, but overall the differences are small. We also observe that Graph Neural Network is more stable compared with Graph Filters. Furthermore, the Graph Neural Networks and Graph Filters with Lipschitz continuous and integral Lipschitz continuous filters are more stable. Both of these observations validate our stability results. 

\begin{figure}[h]
  \centering
  \includegraphics[height=4.5cm,width=6.5cm]{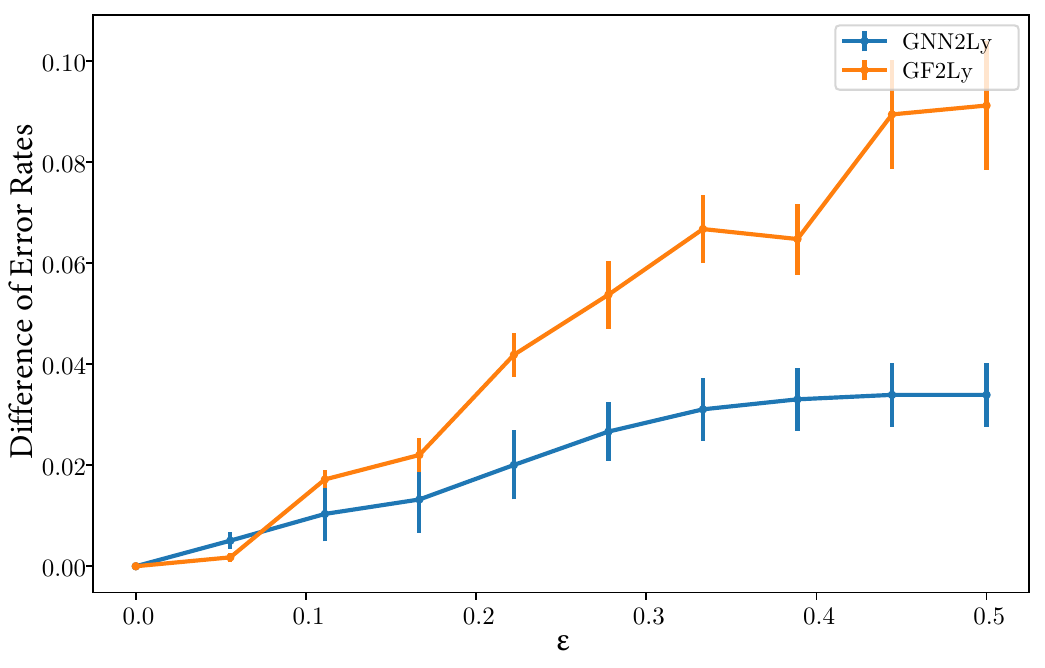}
\caption{Difference between error rates on the original test dataset and the deformed one. }
\label{fig:sim}
\end{figure}

\begin{figure}[h]
  \centering
  \includegraphics[height=4.5cm,width=6.5cm]{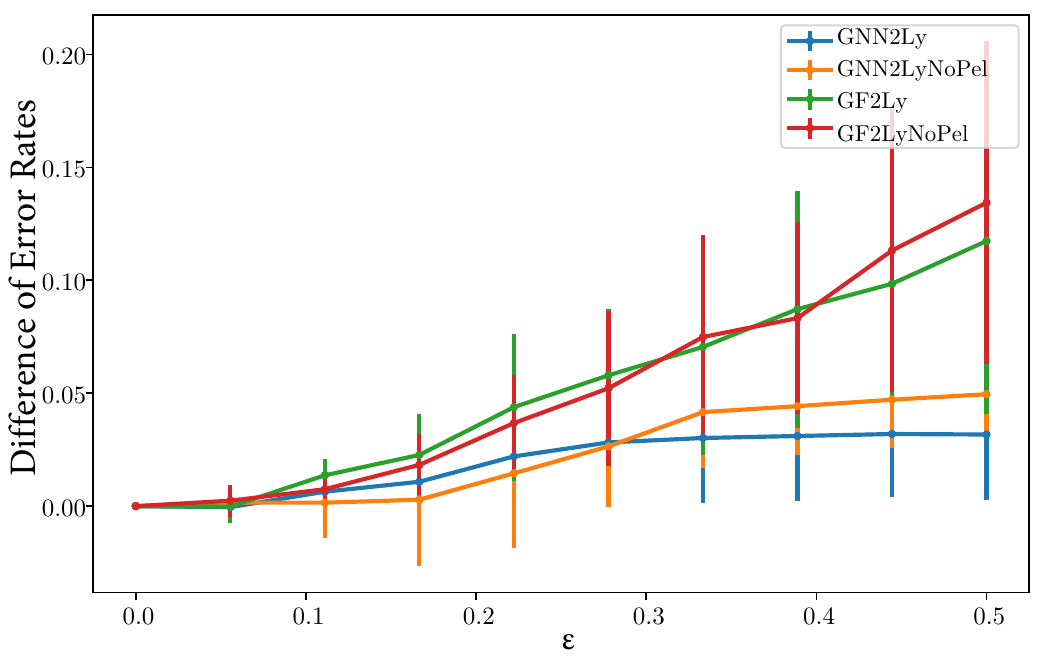}
\caption{Difference between error rates on the original test dataset and the deformed one. }
\label{fig:sim}
\end{figure}
To further verify the discriminability under perturbations, we trained and tested the architectures with perturbed dataset. We can see from Table \ref{tb:results-perturb} that both GNN and GF can identify the chair model with small error rates while the error rates grow slightly with the increase of perturbation levels. GNNs still outperform GFs in discriminablity with the help of nonlinearity.

\begin{table}[h]
\centering
\begin{tabular}{l|c |c  } \hline
Architecture    & $\epsilon = 0.2$ & 0.4   \\ \hline
GNN2Ly		& $7.37\% \pm 1.43\%$ &  $7.71\% \pm 3.96\%$ \\\hline
GF2Ly	& $13.76\% \pm 6.82\%$  & $13.54\% \pm 7.16\%$  \\ \hline\hline
Architecture    & $\epsilon = 0.6$ & 0.8   \\ \hline
GNN2Ly& $8.04\% \pm 2.83\%$ & $11.01\% \pm 6.33\%$  \\ \hline
GF2Ly	& $14.76\% \pm 5.67\%$ & $16.04\% \pm 6.34\%$ \\ \hline
\end{tabular}
\caption{Classification error rates for model `chair' with perturbed training and test dataset. Average over 5 data realizations. The number of nodes is $n=300$.}
\label{tb:results-perturb}
\vspace{-3mm}
\end{table} 

{We further look into the stability of GNNs with respect to the size of the graph with different perturbation levels. We study the stability of a 2-layer GNN with Lipschitz and integral Lipschitz continuous filters on graphs with a maximum of $n=896$ points. Then we randomly sub-sample to generate graphs with the number of nodes $n = 128, 256,...,896$. We plot the difference of the error rates achieved by the GNN between the original point cloud models and the perturbed point cloud models. The position of each point in the perturbed cloud is displaced by $\epsilon/ \sqrt[3]{n}$. Since this is a point cloud in three dimensions, the normalization by $1 / \sqrt[3]{n}$ is such that the gradient of the deformation stays constant across different sub-samplings. 
In Figure \ref{fig:sim-stability}, we show the stability of a 2-layer GNN with Lipschitz and integral Lipschitz continuous filters as we vary the number of points in the point cloud. We observe from the figure that the stability bound does not grow with the number of nodes.
This experiment shows that MNN stability bounds are a more accurate model of the behavior of GNNs when graphs are sampled from a manifold while generic stability bounds predict a stability error that can grow with $\sqrt{n}$ \cite{gama2020stability}. We emphasize that this growing stability bounds of \cite{gama2020stability} are due to the fact that the graph and perturbation model are generic. Our bounds are tighter with a specific assumption on the graph -- it is sampled from a manifold -- and the perturbation model -- it is a bounded manifold deformation with bounded gradient norm.}

\begin{figure}[h]
  \centering
  \includegraphics[height=4.5cm]{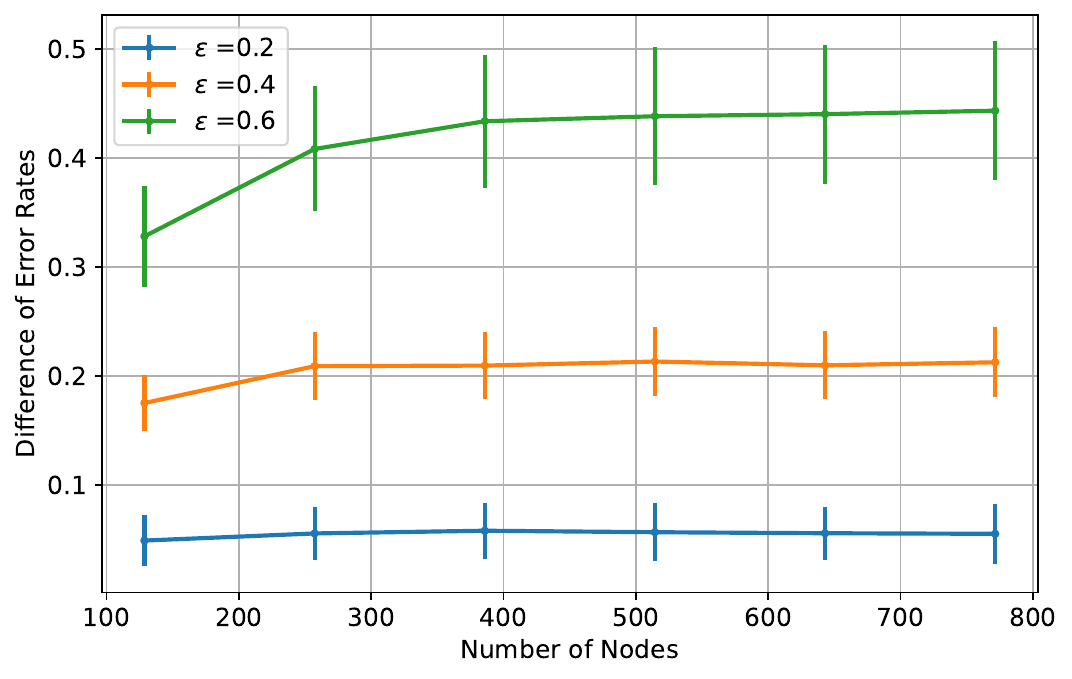}
\caption{Difference between error rates on the original test dataset and the deformed one on growing size of graphs. }
\label{fig:sim-stability}
\end{figure}

\section{Conclusions}
\label{sec:conclusion}
In this paper, we have defined manifold convolutions and manifold neural networks. We prove that the deformations on the embedded submanifolds can be represented as a form of perturbations to the Laplace-Beltrami operator. Considering the infinite dimensionality of LB operators, we import the definition of frequency difference threshold filters and frequency ratio threshold filters to help separate the spectrum. By assigning similar frequency responses to the eigenvalues that are close enough, these filters can be proved to be stable under absolute and relative perturbations to the LB operator respectively with Lipschitz continuous assumptions. While the manifold filters need to trade-off between the stability and discriminability.  MNNs composed with layers of manifold filters and pointwise nonlinearities can be proved to be stable to absolute and relative perturbations to the LB operators. While the frequency mixing brought by pointwise nonlinearity can help with the discriminability. We conclude that the MNNs are thus both stable to deformations and discriminative. We also show the discretizations of MNNs in both spatial and time domains to make our proposed MNNs implementable. 
We finally verified our results numerically with a point cloud classification problem with ModelNet10 dataset.

\urlstyle{same}
\bibliographystyle{IEEEtran}
\bibliography{references}

\appendix
 {\section{Appendix}


\subsection{Proof of Theorem \ref{thm:perturb}}
 \label{app:perturb}
 We first introduce the concepts of the tangent map and the Jacobian matrix \cite{gross2023manifolds}.
 \begin{definition}
     Let $\ccalM$ be the manifold and $\tau:\ccalM\rightarrow \ccalM$. For any $x\in \ccalM$, we define the tangent map to be the linear map 
     \begin{equation}
         \tau_{*,x} :T_x\ccalM \rightarrow T_{\tau(x)}\ccalM.
     \end{equation}
When $\tau$ is a smooth map and $\ccalM$ is an embedded manifold in $\reals^\mathsf{N}$, for $x\in\ccalM$, the tangent space $T_x\ccalM$ can be canonically identified with $\reals^\mathsf{N}$, using the basis 
\begin{equation}
    \frac{\partial}{\partial x^1}\Bigr|_x, \frac{\partial}{\partial x^2}\Bigr|_x\cdots \frac{\partial}{\partial x^\mathsf{N}}\Bigr|_x \in T_x\ccalM,
\end{equation}
similarly the tangent space $T_{\tau(x)}\ccalM$ is equipped with the basis
\begin{equation}
    \frac{\partial}{\partial y^1}\Bigr|_{\tau(x)}, \frac{\partial}{\partial y^2}\Bigr|_{\tau(x)}\cdots \frac{\partial}{\partial y^\mathsf{N}}\Bigr|_{\tau(x)} \in T_{\tau(x)}\ccalM,
\end{equation}
and is just the Jacobian matrix $J_x(\tau)$ of $\tau$ at $x$,
\begin{equation}
    \tau_*\left( \frac{\partial}{\partial x^i} \Bigr|_x \right) = \sum_{j=1}^\mathsf{N} \frac{\partial \tau^j}{\partial x^i}\Bigr|_x \cdot \frac{\partial }{\partial y^j}\Bigr|_{\tau(x)},
\end{equation}
where $\frac{\partial \tau^j}{\partial x^i}\Bigr|_x$ is the entry of the Jacobian matrix $[J_x(\tau)]_{i,j}$ and $\tau^j$ is the $j$-th component of $\tau$.
 \end{definition}
 
 With the tangent map $\tau_{*,x}$  as the linear map between the tangent spaces $T_x\ccalM$ to $T_{\tau(x)}\ccalM$ and
 based on the definition of $\ccalL'$ in equation \eqref{eqn:deform} together with the definition of Laplace-Beltrami operator in \eqref{eqn:Laplacian}, the operation carried out on the deformed manifold data $f$ can be written as
\begin{align}
   - \ccalL' f(x)&= (\nabla\cdot \nabla)  f(\tau(x))
   \\ \label{eqn:changevariable} &=   (J_x(\tau)^T \nabla_\tau \cdot J_x(\tau)^T \nabla_\tau)  f(\tau(x)).
\end{align}
The equality in \eqref{eqn:changevariable} results from the chain rule of gradient operator where $\nabla_\tau$ is denoted as the intrinsic gradient around $\tau(x)$ in the tangent space $T_{\tau(x)}\ccalM$.
By replacing $J_x(\tau) = I + \Delta_x$ the inner product term,  \eqref{eqn:changevariable} can be rewritten as
\begin{align}
    \nonumber J_x(\tau)^T\nabla_\tau\cdot  J_x(\tau)^T \nabla_\tau  =  \nabla_\tau \cdot & \nabla_\tau  + 2(\Delta_x^T \nabla_\tau \cdot \nabla_\tau ) \\
    &+  \Delta_x^T\nabla_\tau \cdot \Delta_x^T\nabla_\tau.
\end{align}
With $\ccalL=- \nabla_\tau \cdot \nabla_\tau$, the perturbed operator is
\begin{align}
   \label{eqn:E1} \ccalL-\ccalL'&= 2(\Delta_x^T \nabla_\tau \cdot \nabla_\tau )  + \Delta_x^T\nabla_\tau \cdot \Delta_x^T\nabla_\tau \\
   \label{eqn:E2} & = 2\|\Delta_x\|_F (\nabla_\tau \cdot \nabla_\tau) + \|\Delta_x\|_F^2 (\nabla_\tau \cdot \nabla_\tau)+\ccalA.
\end{align}
From \eqref{eqn:E1} to \eqref{eqn:E2}, we extract the relative term and use $\ccalA$ to represent the compliment terms.
This leads to $\bbE  = \|\Delta_x\|_F^2 + 2\|\Delta_x\|_F$,
as the relative perturbation term, the norm of which is bounded by the leading term as $O(\epsilon)$. The norm of the compliment term therefore can be written as
\vspace{-1mm}
    \begin{align}
    \|\ccalA\|& = \|\bbE (\nabla_\tau \cdot \nabla_\tau)- 2(\Delta_x^T \nabla_\tau \cdot \nabla_\tau )  - \Delta_x^T\nabla_\tau \cdot \Delta_x^T\nabla_\tau \|\\
   \nonumber &\leq \left\lVert2\|\Delta_x\|_F (\nabla_\tau \cdot \nabla_\tau) -  2(\Delta_x^T \nabla_\tau \cdot \nabla_\tau )\right\rVert\\
    &\qquad \quad +\left\lVert \|\Delta_x\|_F^2 (\nabla_\tau \cdot \nabla_\tau) - \Delta_x^T\nabla_\tau \cdot \Delta_x^T\nabla_\tau\right\rVert,
\end{align}
which can be also bounded by the leading terms as $O(\epsilon)$ combining with the boundedness of the gradient field.

\subsection{Lemmas}
\label{app:lemmas}
Now we need to include two important lemmas to analyze the influence on eigenvalues and eigenfunctions caused by the perturbation.
\begin{lemma}\label{lem:eigenvalue_absolute}[Weyl's Theorem]
The eigenvalues of LB operators $\ccalL$ and perturbed $\ccalL'=\ccalL+\ccalA$ satisfy
\begin{equation}
|\lambda_i-\lambda'_i|\leq \|\ccalA\|, \text{ for all }i=1,2\hdots
\end{equation}
\end{lemma}
\begin{proof}[Proof of Lemma \ref{lem:eigenvalue_absolute}]

The minimax principle asserts that
\begin{align}
    \lambda_i(\ccalL)=\max_{codim T = i-1}\lambda[\ccalL, T]=\max_{codim T \leq i-1} \min_{u\in T, \|u\|=1} \langle \ccalL u, u \rangle.
\end{align}

Then for any $1\leq i $, we have
\begin{align}
    \lambda_i(\ccalL') &=\max_{codim T\leq i-1} \min_{ u\in T, \|u\|=1}  \langle (\ccalL+\ccalA) u, u\rangle \\
    & = \max_{codim T\leq i-1} \min_{ u\in T, \|u\|=1} \left( \langle  (\ccalL  u, u\rangle + \langle \ccalA u, u\rangle   \right)\\
    & \geq \max_{codim T\leq i-1} \min_{  u\in T, \|u\|=1}  \left\langle  \ccalL  u, u\rangle   + \lambda_1(\ccalA) \right)\\
    & = \lambda_1(\ccalA)+ \max_{codim T\leq i-1} \min_{  u\in T, \|u\|=1 } \langle  \ccalL  u, u\rangle  \\
    & = \lambda_i(\ccalL)+\lambda_1(\ccalA).
\end{align}
Similarly, we can have $\lambda_i(\ccalL') \leq \lambda_i(\ccalL)+ \max_k\lambda_k(\ccalA)$. This leads to $\lambda_1(\ccalA)\leq \lambda_i(\ccalL' )-\lambda_i(\ccalL) \leq \max_k\lambda_k(\ccalA)$. This leads to the conclusion that:
\begin{equation}
    |\lambda'_i-\lambda_i|\leq \|\ccalA\|.
\end{equation}
\end{proof}

To measure the difference of eigenfunctions, we introduce the Davis-Kahan $\sin\theta$ theorem as follows.
\begin{lemma}\label{lem:davis-kahan}[Davis-Kahan $\sin\theta$ Theorem]
Suppose the spectra of operators $\ccalL$ and $\ccalL'$ are partitioned as $\sigma\bigcup\Sigma$ and $\omega\bigcup \Omega$ respectively, with $\sigma\bigcap \Sigma=\emptyset$ and $\omega\bigcap\Omega=\emptyset$. Then we have
\begin{equation}
\|E_\ccalL(\sigma)-E_{\ccalL'}(\omega)\|\leq \frac{\pi}{2}\frac{\|(\ccalL'-\ccalL)E_\ccalL(\sigma)\|}{d}\leq \frac{\pi}{2}\frac{\|\ccalL'-\ccalL\|}{d},
\end{equation}
where $d$ satisfies $\min_{x\in\sigma,y\in\Omega}|x-y|\geq d$ and $\min_{x\in\Sigma,y\in\omega}|x-y|\geq d$.
\end{lemma}
\begin{proof}[Proof of Lemma \ref{lem:davis-kahan}] 
See \cite{seelmann2014notes}.
\end{proof}

\begin{lemma}\label{lem:eigenvalue_relative}
The eigenvalues of LB operators $\ccalL$ and perturbed $\ccalL'=\ccalL+\bbE\ccalL$ with $\|\bbE\|\leq  \epsilon$ satisfy
\begin{align}
    |\lambda_i-\lambda'_i|\leq \epsilon |\lambda_i|, \text{ for all }i=1,2\hdots
\end{align}
\end{lemma}
\begin{proof}[Proof of Lemma \ref{lem:eigenvalue_relative}]
With the assumption that $\ccalL'=\ccalL+\bbE\ccalL$, we have
\begin{align}
    \lambda_i(\ccalL + \bbE\ccalL)& = \max_{codim T\leq i-1} \min_{ u\in T, \|u\|=1}  \langle (\ccalL+\bbE\ccalL) u, u\rangle \\
    & = \max_{codim T\leq i-1} \min_{u\in T, \|u\|=1} \left(  \langle \ccalL u, u\rangle   +  \langle \bbE\ccalL u, u\rangle  \right)\\
    & =\lambda_i(\ccalL) + \max_{codim T\leq i-1} \min_{u\in T, \|u\|=1} \langle \bbE\ccalL u, u\rangle.
\end{align}
For the second term, we have
\begin{align}
    |\langle \bbE\ccalL u, u\rangle| &\leq \langle |\bbE|  |\ccalL|u, u \rangle   \leq \epsilon \sum_i |\lambda_i(\ccalL)||\xi_i|^2 = \epsilon \langle |\ccalL| u,u \rangle
\end{align}
Therefore, we have
\begin{align}
 & \nonumber \lambda_i(\ccalL+\bbE\ccalL) \leq \lambda_i(\ccalL) + \epsilon
   \max_{codim T\leq i-1} \min_{u\in T, \|u\|=1} \langle |\ccalL| u, u\rangle\\
   &\qquad \qquad\quad  = \lambda_i(\ccalL) + \epsilon |\lambda_i(\ccalL)|,\\
   &\lambda_i(\ccalL + \bbE\ccalL) \geq \lambda_i(\ccalL) -\epsilon |\lambda_i(\ccalL)|,\\
   &\lambda_i(\ccalL)-\epsilon |\lambda_i(\ccalL)|\leq  \lambda_i(\ccalL + \bbE\ccalL)\leq \lambda_i(\ccalL) +\epsilon|\lambda_i(\ccalL)|,
\end{align}
which concludes the proof.
\end{proof}


\subsection{Proof of Theorem \ref{thm:stability_abs_filter}}
\label{app:stability_abs_filter}
In the following, we denote $\langle\cdot,\cdot \rangle_{L^2(\ccalM)}$ as $\langle\cdot,\cdot \rangle$ and $\|\cdot\|_{L^2(\ccalM)}$ as $\|\cdot\|$ for simplicity. We start by bounding the norm difference between two outputs of filter functions on operators $\ccalL$ and $\ccalL'$ defined in \eqref{eqn:convolution-general} as
\begin{align}
 \nonumber  &\left\| \bbh(\ccalL) f - \bbh(\ccalL')f\right\| 
  = \\
  &\qquad \qquad\left\| \sum_{i=1}^\infty \hat{h}(\lambda_{i}) \langle f, \bm\phi_i \rangle \bm\phi_i - \sum_{i=1}^\infty \hat{h}(\lambda'_{i}) \langle f, \bm\phi'_i \rangle \bm\phi'_i \right\|. \label{eqn:diff}
\end{align}
We denote the index of partitions that contain a single eigenvalue as a set $\ccalK_s$ and the rest as a set $\ccalK_m$. We can decompose the filter function as $\hat{h}(\lambda)=h^{(0)}(\lambda)+\sum_{l\in\ccalK_m}h^{(l)}(\lambda)$ with
\begin{align}
\label{eqn:h0}& h^{(0)}(\lambda) = \left\{ 
\begin{array}{cc} 
              \hat{h}(\lambda)-\sum\limits_{l\in\ccalK_m}\hat{h}(C_l)  &  \lambda\in[\Lambda_k(\alpha)]_{k\in\ccalK_s} \\
                0& \text{otherwise}  \\
                \end{array} \right. \\
\label{eqn:hl}& h^{(l)}(\lambda) = \left\{ 
\begin{array}{cc} 
                \hat{h}(C_l) &  \lambda\in[\Lambda_k(\alpha)]_{k\in\ccalK_s} \\
                \hat{h}(\lambda) & 
                \lambda\in\Lambda_l(\alpha)\\
                0 &
                \text{otherwise}  \\
                \end{array} \right.             
\end{align}
where $C_l$ is some constant in $\Lambda_l(\alpha)$. We can start by analyzing the output difference of $h^{(0)}(\lambda)$. With the triangle inequality, the norm difference can then be written as
\begin{align}
 & \nonumber \left\| \sum_{i=1}^\infty h^{(0)}(\lambda_{i}) \langle f, \bm\phi_i \rangle \bm\phi_i  -  h^{(0)}(\lambda'_i )  \langle f, \bm\phi'_i \rangle \bm\phi'_i \right\|  \\
 &\nonumber =\Bigg\|\sum_{i=1}^\infty  h^{(0)}(\lambda_{i}) \langle f, \bm\phi_i \rangle \bm\phi_i -  h^{(0)}(\lambda_{i}) \langle f, \bm\phi'_i \rangle \bm\phi'_i  + \\
 &\qquad \qquad \qquad h^{(0)}(\lambda_{i}) \langle f, \bm\phi'_i \rangle \bm\phi'_i- h^{(0)}(\lambda'_{i}) \langle f, \bm\phi'_i \rangle \bm\phi'_i \Bigg\| \\
  & \nonumber \leq \left\|\sum_{i=1}^\infty  h^{(0)}(\lambda_i)  \langle f, \bm\phi_i \rangle \bm\phi_i -  h^{(0)}(\lambda_i ) \langle f, \bm\phi'_i \rangle \bm\phi'_i\right\|  +\\
  &\qquad\quad  \left\|\sum_{i =1}^\infty h^{(0)}(\lambda_{i}) \langle f, \bm\phi'_i \rangle \bm\phi'_i -  h^{(0)}(\lambda'_{i}) \langle f, \bm\phi'_i \rangle \bm\phi'_i \right\| \\
    &\nonumber \leq \Bigg\| \sum_{i=1}^\infty  h^{(0)}(\lambda_i)\Bigg( \langle f, \bm\phi_i \rangle\bm\phi_i-\langle f, \bm\phi_i \rangle\bm\phi'_i +\langle f, \bm\phi_i \rangle\bm\phi'_i - \\
    &\quad \langle f, \bm\phi'_i \rangle \bm\phi'_i \Bigg) \Bigg\| +\left\|\sum_{i =1}^\infty  (h^{(0)}(\lambda_i ) -h^{(0)}(\lambda'_i) ) \langle f, \bm\phi'_i \rangle \bm\phi'_i  \right\| \\
    & \nonumber \leq \left\| \sum_{i=1}^\infty h^{(0)}(\lambda_i )\langle f, \bm\phi_i \rangle (\bm\phi_i - \bm\phi'_i ) \right\| \\\nonumber &\qquad \qquad  +\left\|  \sum_{i =1}^\infty  h^{(0)}(\lambda_i )\langle f, \bm\phi_i - \bm\phi'_i  \rangle \bm\phi'_i \right\|\\
      &\qquad \qquad \quad 
     + \left\|\sum_{i=1}^\infty  (h^{(0)}(\lambda_i ) -h^{(0)}(\lambda'_i) ) \langle f, \bm\phi'_i \rangle \bm\phi'_i  \right\| \label{eqn:3}
\end{align}

For the first term in \eqref{eqn:3}, we employ Lemma \ref{lem:davis-kahan} and therefore we have $\sigma=\lambda_i$ and $\omega=\lambda'_i$, for $\lambda_i\in [\Lambda_k(\alpha)]_{k\in\ccalK_s}$ we can have
\begin{align}
\left\| \bm\phi_i -\bm\phi'_i \right\| \leq \frac{\pi}{2} \frac{\|\ccalA\|}{\alpha-\epsilon}= \frac{\pi}{2} \frac{\epsilon}{\alpha-\epsilon}.
\end{align}
Here $d$ can be seen as $d=\min_{\lambda_i\in\Lambda_k(\alpha),\lambda_j\in\Lambda_l(\alpha),k\neq l}|\lambda_i-\lambda_j'|$. Combined with the fact that $|\lambda_i-\lambda_j|>\alpha$ and $|\lambda_i-\lambda_i'|\leq \epsilon$ for all $\lambda_i\in\Lambda_k(\alpha),\lambda_j\in\Lambda_l(\alpha),k\neq l$, we have $d\geq \alpha-\epsilon$. With Cauchy-Schwartz inequality, we have the first term in \eqref{eqn:3} bounded as
\begin{align}
&\nonumber\left\| \sum_{i=1}^\infty h^{(0)}(\lambda_i )\langle f, \bm\phi_i \rangle (\bm\phi_i - \bm\phi'_i ) \right\|\\
& \leq \sum_{i=1}^\infty |h^{(0)}(\lambda_i)| | \langle f, \bm\phi_i \rangle | \left\|\bm\phi_i-\bm\phi'_i \right\| \leq  \frac{N_s\pi\epsilon}{2(\alpha-\epsilon)}  \|f\|.
\end{align}

The second term in \eqref{eqn:3} is bounded as
\begin{align}
 &\nonumber \left\|  \sum_{i =1}^\infty  h^{(0)}(\lambda_i )\langle f, \bm\phi_i - \bm\phi'_i  \rangle \bm\phi'_i \right\| \\
 &\leq   \sum_{i =1}^\infty |h^{(0)}(\lambda_i)| \|\bm\phi_i - \bm\phi'_i \| \|f\| \leq   \frac{N_s \pi\epsilon}{2(\alpha-\epsilon)}  \|f\|.
\end{align}
These two bounds are obtained by noting that $|h^{(0)}(\lambda)|<1$ and $h^{(0)}(\lambda)=0$ for $\lambda\in[\Lambda_k(\alpha)]_{k\in\ccalK_m}$. The number of eigenvalues within $[\Lambda_k(\alpha)]_{k\in\ccalK_s}$ is denoted as $N_s$. The third term in \eqref{eqn:3} can be bounded by the Lipschitz continuity of $h$ combined with Lemma \ref{lem:eigenvalue_absolute}.
\begin{align}
\nonumber  \Bigg\|\sum_{i=1}^\infty  &(h^{(0)}(\lambda_i ) -h^{(0)}(\lambda'_i) ) \langle f, \bm\phi'_i \rangle \bm\phi'_i  \Bigg\|^2 
  \\ \nonumber & \leq \sum_{i=1}^\infty | h^{(0)}(\lambda_{i}) -h^{(0)}(\lambda'_i) |^2 |\langle f, \bm\phi'_i \rangle|^2 \\
  &\leq \sum_{i =1}^\infty A_h^2 |\lambda_i - \lambda'_i |^2 |\langle f, \bm\phi'_i \rangle|^2 \leq  A_h^2 \epsilon^2 \|f\|^2.
\end{align}

Then we need to analyze the output difference of $h^{(l)}(\lambda)$, we can bound this as
\begin{align}
    \nonumber &\left\| \bbh^{(l)}(\ccalL)f -\bbh^{(l)}(\ccalL')f \right\| 
    \\& \leq \left\| (\hat{h}(C_l)+\delta)f -(\hat{h}(C_l)-\delta)f\right\| \leq 2\delta\|f\|,
\end{align}
where $\bbh^{(l)}(\ccalL)$ and $\bbh^{(l)}(\ccalL')$ are manifold filters with filter function $h^{(l)}(\lambda)$ on the LB operators $\ccalL$ and $\ccalL'$ respectively.
Combining the filter functions, we can write
\begin{align}
   \nonumber &\|\bbh(\ccalL)f-\bbh(\ccalL')f\|=\\&
    \left\|\bbh^{(0)}(\ccalL)f +\sum_{l\in\ccalK_m}\bbh^{(l)}(\ccalL)f - \bbh^{(0)}(\ccalL')f - \sum_{l\in\ccalK_m} \bbh^{(l)}(\ccalL')f \right\|\\
    &\leq \|\bbh^{(0)}(\ccalL)f-\bbh^{(0)}(\ccalL')f\|+\sum_{l\in\ccalK_m}\|\bbh^{(l)}(\ccalL)f-\bbh^{(l)}(\ccalL')f\|\\
    &\label{eqn:sta-filter-alpha}\leq \frac{N_s\pi\epsilon}{\alpha-\epsilon}\|f\| + A_h\epsilon\|f\| +2(N-N_s)\delta\|f\|,
\end{align}
which concludes the proof.

\subsection{Proof of Theorem \ref{thm:stability_nn}}
\label{app:stability_nn}
To bound the output difference of MNNs, we need to write in the form of features of the final layer
 \begin{equation}
 \|\bm\Phi(\bbH,\ccalL,f)-\bm\Phi(\bbH,\ccalL',f)\| =  \left\| \sum_{q=1}^{F_L} f_L^q - \sum_{q=1}^{F_L} f_L^{'q}\right\|.
 \end{equation}
The output signal of layer $l$ of MNN $\bbPhi(\bbH,\ccalL, f)$ can be written as
\begin{equation}
 f_l^p = \sigma\left( \sum_{q=1}^{F_{l-1}} \bbh_l^{pq}(\ccalL) f_{l-1}^q\right).
\end{equation}
Similarly, for the perturbed $\ccalL'$ the corresponding MNN is $\bbPhi(\bbH,\ccalL',f)$ the output signal can be written as
 \begin{equation}
 f_l^{'p} = \sigma\left( \sum_{q=1}^{F_{l-1}} \bbh_l^{pq}(\ccalL') f_{l-1}^{'q}\right).
 \end{equation}
The difference therefore becomes
 \begin{align}
 &\nonumber\| f_l^p - f_l^{'p} \| \\& =\left\|  \sigma\left( \sum_{q=1}^{F_{l-1}} \bbh_l^{pq}(\ccalL) f_{l-1}^q\right) -  \sigma\left( \sum_{q=1}^{F_{l-1}} \bbh_l^{pq}(\ccalL') f_{l-1}^{'q}\right) \right\|.   
 \end{align}
With the assumption that $\sigma$ is normalized Lipschitz, we have
 \begin{align}
  \| f_l^p - f_l^{'p} \| &\leq \left\| \sum_{q=1}^{F_{l-1}}  \bbh_l^{pq}(\ccalL) f_{l-1}^q - \bbh_l^{pq}(\ccalL') f_{l-1}^{'q}  \right\| \\&\leq \sum_{q=1}^{F_{l-1}} \left\|  \bbh_l^{pq}(\ccalL) f_{l-1}^q - \bbh_l^{pq}(\ccalL') f_{l-1}^{'q} \right\|.
 \end{align}
By adding and subtracting $\bbh_l^{pq}(\ccalL') f_{l-1}^{q}$ from each term, combined with the triangle inequality we can get
 \begin{align}
 & \nonumber \left\|  \bbh_l^{pq}(\ccalL) f_{l-1}^q - \bbh_l^{pq}(\ccalL') f_{l-1}^{'q} \right\| \\\nonumber &\quad \leq \left\|  \bbh_l^{pq}(\ccalL) f_{l-1}^q - \bbh_l^{pq}(\ccalL') f_{l-1}^{q} \right\| \\&\qquad \qquad \qquad + \left\| \bbh_l^{pq}(\ccalL') f_{l-1}^q - \bbh_l^{pq}(\ccalL') f_{l-1}^{'q} \right\|
 \end{align}
The first term can be bounded with \eqref{eqn:sta-filter-alpha} for absolute perturbations. The second term can be decomposed by Cauchy-Schwartz inequality and non-amplifying of the filter functions as
 \begin{align}
 \left\| f_{l}^p - f_l^{'p} \right\| \leq \sum_{q=1}^{F_{l-1}} C_{per} \epsilon \| f_{l-1}^q\| + \sum_{q=1}^{F_{l-1}} \| f_{l-1}^q - f_{l-1}^{'q} \|,
 \end{align}
where $C_{per}$ representing the constant in the stability bound of manifold filters. To solve this recursion, we need to compute the bound for $\|f_l^p\|$. By normalized Lipschitz continuity of $\sigma$ and the fact that $\sigma(0)=0$, we can get
 \begin{align}
 \nonumber &\| f_l^p \|\leq \left\| \sum_{q=1}^{F_{l-1}} \bbh_l^{pq}(\ccalL) f_{l-1}^{q}  \right\| \leq  \sum_{q=1}^{F_{l-1}}  \left\| \bbh_l^{pq}(\ccalL)\right\|  \|f_{l-1}^{q}  \| \\
 &\qquad \leq   \sum_{q=1}^{F_{l-1}}   \|f_{l-1}^{q}  \| \leq \prod\limits_{l'=1}^{l-1} F_{l'} \sum_{q=1}^{F_0}\| f^q \|.
 \end{align}
 Insert this conclusion back to solve the recursion, we can get
 \begin{align}
 \left\| f_{l}^p - f_l^{'p} \right\| \leq l C_{per}\epsilon \left( \prod\limits_{l'=1}^{l-1} F_{l'} \right) \sum_{q=1}^{F_0} \|f^q\|.
 \end{align}
 Replace $l$ with $L$ we can obtain
 \begin{align}
 &\nonumber \|\bm\Phi(\bbH,\ccalL,f) - \bm\Phi(\bbH,\ccalL',f)\| \\
 &\qquad \qquad \leq \sum_{q=1}^{F_L} \left( L C_{per}\epsilon \left( \prod\limits_{l'=1}^{L-1} F_{l'} \right) \sum_{q=1}^{F_0} \|f^q\| \right).
 \end{align}
 With $F_0=F_L=1$ and $F_l=F$ for $1\leq l\leq L-1$, then we have
  \begin{align}
 \|\bm\Phi(\bbH,\ccalL,f) - \bm\Phi(\bbH,\ccalL',f)\| \leq LF^{L-1} C_{per}\epsilon \|f\|,
 \end{align}
which concludes the proof.

\subsection{Proof of Theorem \ref{thm:stability_rela_filter}}
\label{app:stability_rela_filter}
The decomposition follows the same routine as \eqref{eqn:diff} shows. 
By decomposing the filter function as \eqref{eqn:h0-gamma} and \eqref{eqn:hl-gamma}, the norm difference can also be bounded separately. 
\begin{align}
\label{eqn:h0-gamma}& h^{(0)}(\lambda) = \left\{ 
\begin{array}{cc} 
                \hat{h}(\lambda)-\sum\limits_{l\in\ccalK_m}\hat{h}(C_l)  &  \lambda\in[\Lambda_k(\gamma)]_{k\in\ccalK_s} \\
                0& \text{otherwise}  \\
                \end{array} \right.  \\
\label{eqn:hl-gamma}& h^{(l)}(\lambda) = \left\{ 
\begin{array}{cc} 
                \hat{h}(C_l) &  \lambda\in[\Lambda_k(\gamma)]_{k\in\ccalK_s} \\
                \hat{h}(\lambda) & 
                \lambda\in\Lambda_l(\gamma)\\
                0 &
                \text{otherwise}  \\
                \end{array} \right.             
\end{align}
where now $\hat{h}(\lambda)=h^{(0)}(\lambda)+\sum_{l\in\ccalK_m}h^{(l)}(\lambda)$ with $\ccalK_s$ defined as the group index set of singletons and $\ccalK_m$ the set of partitions that contain multiple eigenvalues. For manifold filter $\bbh^{(0)}(\ccalL)$ with filter function $h^{(0)}(\lambda)$, the norm difference can also be written as
\begin{align}
 \label{eqn:rela-h0}  & \nonumber \left\| \sum_{i=1}^\infty h^{(0)}(\lambda_{i}) \langle f, \bm\phi_i \rangle \bm\phi_i  -  h^{(0)}(\lambda'_i )  \langle f, \bm\phi'_i \rangle \bm\phi'_i \right\| \\
  & \nonumber \leq \left\| \sum_{i=1}^\infty h^{(0)}(\lambda_i )\langle f, \bm\phi_i \rangle (\bm\phi_i - \bm\phi'_i ) \right\| \\ \nonumber&\qquad\qquad  + \Bigg\|  \sum_{i =1}^\infty  h^{(0)}(\lambda_i )\langle f, \bm\phi_i - \bm\phi'_i  \rangle \bm\phi'_i \Bigg\|\\& \qquad\qquad \quad+ \left\|\sum_{i=1}^\infty  (h^{(0)}(\lambda_i ) -h^{(0)}(\lambda'_i) ) \langle f, \bm\phi'_i \rangle \bm\phi'_i  \right\| . 
\end{align}
The difference of the eigenvalues due to relative perturbations can be similarly addressed by Lemma \ref{lem:eigenvalue_relative}.

The first two terms of \eqref{eqn:rela-h0} rely on the differences of eigenfunctions, which can be derived with Davis-Kahan Theorem in Lemma \ref{lem:davis-kahan}, the difference of eigenfunctions can be written as
\begin{align}
\| \bbE\ccalL \bm\phi_i \| =\| \bbE\lambda_i\bm\phi_i \|=\lambda_i \|\bbE \bm\phi_i\|\leq\lambda_i\|\bbE\|\|\bm\phi_i\|\leq \lambda_i \epsilon.
\end{align}
The first term in \eqref{eqn:rela-h0} then can be bounded as
\begin{align}
&\nonumber\left\| \sum_{i=1}^\infty h^{(0)}(\lambda_i )\langle f, \bm\phi_i \rangle (\bm\phi_i - \bm\phi'_i ) \right\|\\
& \leq \sum_{i=1}^\infty |h^{(0)}(\lambda_i)| | \langle f, \bm\phi_i \rangle | \left\|\bm\phi_i-\bm\phi'_i \right\| \leq \sum_{i\in\ccalK_s} \frac{\pi\lambda_i \epsilon}{2d_i}  \|f\|.
\end{align} 
Because $d_i=\min\{ |\lambda_i-\lambda'_{i-1}|, |\lambda'_i-\lambda_{i-1}|, |\lambda'_{i+1}-\lambda_i| , | \lambda_{i+1}-\lambda'_i|\}$, with Lemma \ref{lem:eigenvalue_relative} implied, we have
\begin{gather}
|\lambda_i-\lambda'_{i-1}|\geq | \lambda_i - (1+\epsilon)\lambda_{i-1}|,\\
 |\lambda'_i-\lambda_{i-1}|\geq |(1-\epsilon)\lambda_i-\lambda_{i-1}|,\\ |\lambda'_{i+1}-\lambda_i|\geq | (1-\epsilon)\lambda_{i+1}-\lambda_i|,\\| \lambda_{i+1}-\lambda'_i|\geq |\lambda_{i+1}-(1+\epsilon)\lambda_i|.
\end{gather}
Combine with Lemma \ref{lem:eigenvalue_relative} and Definition \ref{def:frt-spectrum}, $d_i\geq \epsilon\gamma +\gamma-\epsilon$:
\begin{align}
 |(1-\epsilon)\lambda_{i+1}-\lambda_i|
 &\geq |\gamma \lambda_i-\epsilon \lambda_{i+1}|\\
 &=\epsilon \lambda_i\left|1-\frac{\lambda_{i+1}}{\lambda_i}+\frac{\gamma}{\epsilon}-1\right|\\&\geq \lambda_i(\gamma-\epsilon+\gamma\epsilon)
\end{align}
This leads to the bound as
\begin{align}
\left\| \sum_{i=1}^\infty h^{(0)}(\lambda_i )\langle f, \bm\phi_i \rangle (\bm\phi_i - \bm\phi'_i ) \right\| \leq   \frac{M_s \pi \epsilon}{2(\gamma-\epsilon+\gamma\epsilon)} \|f\|.
\end{align}

The second term in \eqref{eqn:rela-h0} can also be bounded as
\begin{align}
    &\nonumber \left\|  \sum_{i =1}^\infty  h^{(0)}(\lambda_i )\langle f, \bm\phi_i - \bm\phi'_i  \rangle \bm\phi'_i \right\| \\
 &\leq   \sum_{i =1}^\infty |h^{(0)}(\lambda_i)| \|\bm\phi_i - \bm\phi'_i \| \|f\|  \leq   \frac{M_s \pi \epsilon}{2(\gamma-\epsilon+\gamma\epsilon)} \|f\|,
\end{align}
which similarly results from the fact that $|h^{(0)}(\lambda)|<1$ and $h^{(0)}(\lambda)=0$ for $\lambda\in[\Lambda_k(\gamma)]_{k\in\ccalK_m}$. The number of eigenvalues within $[\Lambda_k(\gamma)]_{k\in\ccalK_s}$ is denoted as $M_s$.

The third term in \eqref{eqn:rela-h0} is:
\begin{align}
   &\nonumber \Bigg\|\sum_{i=1}^\infty  (h^{(0)}(\lambda_i ) -h^{(0)}(\lambda'_i) ) \langle f, \bm\phi'_i \rangle \bm\phi'_i  \Bigg\|^2 \\
    &\leq  \sum_{i=1}^\infty\left( \frac{B_h \epsilon|\lambda_i|}{(\lambda_i+\lambda_i')/2}\right)^2   \langle f,\bm\phi'_i \rangle^2 \leq \left( \frac{2B_h\epsilon}{2-\epsilon}\right)^2\|f\|^2,
\end{align}
with the use of Lemma \ref{lem:eigenvalue_relative} and Definition \ref{def:int-lipschitz}.

Then we need to analyze the output difference of $h^{(l)}(\lambda)$.
\begin{align}
     \nonumber &\left\| \bbh^{(l)}(\ccalL)f -\bbh^{(l)}(\ccalL')f \right\| 
    \\& \leq \left\| (\hat{h}(C_l)+\delta)f -(\hat{h}(C_l)-\delta)f\right\| \leq 2\delta\|f\|.
\end{align}

Combine the filter function, we could get 
\begin{align}
\label{eqn:sta-filter-gamma}
    \nonumber &\|\bbh(\ccalL)f-\bbh(\ccalL')f\|=\\&
    \left\|\bbh^{(0)}(\ccalL)f +\sum_{l\in\ccalK_m}\bbh^{(l)}(\ccalL)f - \bbh^{(0)}(\ccalL')f - \sum_{l\in\ccalK_m} \bbh^{(l)}(\ccalL')f \right\|\\
    &\leq \|\bbh^{(0)}(\ccalL)f-\bbh^{(0)}(\ccalL')f\|+\sum_{l\in\ccalK_m}\|\bbh^{(l)}(\ccalL)f-\bbh^{(l)}(\ccalL')f\|\\
    &\leq \frac{M_s\pi\epsilon}{\gamma-\epsilon+\gamma\epsilon}\|f\| + \frac{2B_h\epsilon}{2-\epsilon}\|f\| +2(M-M_s)\delta\|f\|,
\end{align}
which concludes the proof.

\subsection{Definition \ref{def:manifold-convolution} and Convolutional Filters in Continuous Time}
\label{app:rem_convolution} 
The manifold convolution in Definition \ref{def:manifold-convolution} can also be motivated with a connection to linear time invariant filters. This requires that we consider the differential equation
\begin{equation} \label{eqn:wave}
   \frac{\partial u(x,t)}{\partial t}=  \frac{\partial}{\partial x} u(x,t) \text{.}
\end{equation}
This is a one-sided wave equation and it is therefore not an exact analogous of the diffusion equation in \eqref{eqn:heat} -- this would require that the second derivative be used in the right of \eqref{eqn:wave}. The important observation to make here is that the exponential of the derivative operator is a time shift so that we can write $u(x,t) = e^{t\partial/\partial x}f(x) = f(x-t)$. This is true because $e^{t\partial/\partial x}f(x)$ and $f(x-t)$ are both solutions of \eqref{eqn:wave}. It then follows that Definition \ref{def:manifold-convolution} particularized to \eqref{eqn:wave} yields the convolution definition
\begin{equation}\label{eqn:conv-1d-1}
    g(x) = \int_{0}^\infty \tdh(t) e^{t\partial/\partial x}f(x)\, \text{d}t.
         = \int_{0}^\infty \tdh(t) f(x-t) \,\text{d}t.
\end{equation}
This is the standard definition of time convolutions. 

The frequency representation result in Proposition \ref{prop:filter-spectral} holds for \eqref{eqn:conv-1d-1} and it implies that standard convolutional filters in continuous time are completely characterized by the frequency response in Definition \ref{def:frequency-response}. The more standard definition of a filter's frequency response as the Fourier transform of the impulse response $\tdh(t)$ -- as opposed to the Laplace transform we use in Definition \ref{def:frequency-response} -- suffices because complex exponentials $e^{jw}$ are an orthonormal basis of eigenfunctions of the derivative operator with associated eigenvalues $j\omega$.

\subsection{Proof of Proposition \ref{prop:finite_num}}
\label{sup:weyl}
Weyl's law in \cite{musser2016weyl} states that if $\ccalM$ is a compact connected oriented Riemannian manifold of dimension $d$ then 
\begin{equation}\label{eqn:weylslaw-musser}
    N(\lambda)\sim \frac{C_d}{(2\pi)^d}Vol(\ccalM) \lambda^{d/2}\text{ with } N(\lambda):=\#\{\lambda_k\leq \lambda\}.
\end{equation}
Since eigenvalues of the LB operator $\ccalL$ are $0<\lambda_1\leq \lambda_2\leq \lambda_3\cdots$  repeated according to its multiplicity, we can have
\begin{equation}\label{eqn:weylslaw}
\lambda_k \sim \frac{(2\pi)^2}{(C_d Vol(\ccalM))^{2/d}}k^{2/d},
\end{equation}
where $C_d$ denotes the volume of the unit ball of $\reals^d$ and $Vol(\ccalM)$ is the volume of manifold $\ccalM$. This indicates that $\lambda_k$ grows with the same order of the magnitude with $\frac{(2\pi)^2}{(C_d Vol(\ccalM))^{2/d}}k^{2/d}$. With this asymptotic equivalence relationship, we can have
\begin{align}
  &  \lambda_{k+1} -\frac{(2\pi)^2(k+1)^{2/d}}{(C_d Vol(\ccalM))^{2/d}} = o\left(\frac{(2\pi)^2(k+1)^{2/d}}{(C_d Vol(\ccalM))^{2/d}}\right),\\
   & \frac{(2\pi)^2}{(C_d Vol(\ccalM))^{2/d}}k^{2/d} - \lambda_k = o(\lambda_k)
\end{align}
Therefore, for any constant $C_1>0$, we can find some $K_1(C_1)>0$, which indicates that $K_1$ depends on $C_1$, such that for all $k>K_1(C_1)$, we have
\begin{equation}
\label{eq:lambda_k+1}
     \lambda_{k+1} -\frac{(2\pi)^2(k+1)^{2/d}}{(C_d Vol(\ccalM))^{2/d}} <  \frac{C_1(2\pi)^2(k+1)^{2/d}}{(C_d Vol(\ccalM))^{2/d}}.
\end{equation}
Similarly, for any constant $C_2>0$, we can find some $K_2(C_2)>0$, such that for all $k>K_2(C_2)$, we have
\begin{equation}
\label{eq:lambda_k}
    \frac{(2\pi)^2}{(C_d Vol(\ccalM))^{2/d}}k^{2/d} - \lambda_k < C_2 \lambda_k.
\end{equation}
Therefore from \eqref{eq:lambda_k+1} and \eqref{eq:lambda_k} we can get upper and lower bound for $\lambda_{k+1}$ and $\lambda_k$ respectively. If 
\begin{align}
   (1+C_1) (k+1)^{2/d}- \frac{k^{2/d}}{1+C_2}\leq \frac{\alpha (Vol(\ccalM) C_d)^{2/d}}{4\pi^2} ,
\end{align}
we can have $\lambda_{k+1}-\lambda_k\leq \alpha$. The left side can be scaled down to $$(k+1)^{2/d}-k^{2/d}\geq \min\{1+C_1, \frac{1}{1+C_2}\}\frac{2}{d}k^{2/d-1} = \frac{C_0}{d}k^{2/d-1}$$ This implies that 
\begin{align}
    k \geq  \left( \frac{\alpha d (Vol(\ccalM) C_d)^{2/d}}{C_04\pi^2} \right)^{\frac{d}{2-d}},
\end{align}
with $d>2$, we can claim that for all $k > K_0(C_0) = \max\{K_1(C_1), K_2(C_2)\}$, if $k$ satisfies
$$k\geq  \Big\lceil \left(\frac{\alpha d}{C_04\pi^2}\right)^{d/(2-d)}(C_d \text{Vol}(\ccalM))^{2/(2-d)} \Big\rceil,$$ it holds that $\lambda_{k+1}-\lambda_k\leq \alpha$. Proof of Proposition \ref{prop:finite_num_rela} is similar and is also based on \eqref{eqn:weylslaw}.


 \subsection{Proof of Proposition \ref{prop:convergence}}
 \label{app:convergence}
 Considering that the discrete points $\{x_1,x_2,\hdots,x_n\}$ are uniformly sampled from manifold $\ccalM$ with measure $\mu$, the empirical measure associated with $\text{d}\mu$ can be denoted as $p_n=\frac{1}{n}\sum_{i=1}^n \delta_{x_i}$, where $\delta_{x_i}$ is the Dirac measure supported on $x_i$. Similar to the inner product defined in the $L^2(\ccalM)$ space \eqref{eqn:innerproduct}, the inner product on $L^2(\bbG_n)$ is denoted as
 \begin{equation}
     \langle u, v\rangle_{L^2(\bbG_n)}=\int u(x)v(x)\text{d}p_n=\frac{1}{n}\sum_{i=1}^n u(x_i)v(x_i).
 \end{equation}
 The norm in $L^2(\bbG_n)$ is therefore $\|u\|^2_{L^2(\bbG_n)} = \langle u, u \rangle_{L^2(\bbG_n)}$, with $u,v \in L^2(\ccalM)$. For signals $\bbu,\bbv \in L^2(\bbG_n)$, the inner product is therefore $\langle \bbu,\bbv \rangle_{L^2(\bbG_n)} = \frac{1}{n}\sum_{i=1}^n [\bbu]_i[\bbv]_i$. From here we write $\|\cdot\|_{L^2(\bbG_n)}$ as $\|\cdot\|$ for simplicity.
 
 We first import the existing results from \cite{belkin2006convergence} which indicates the spectral convergence of the constructed Laplacian operator based on the graph $\bbG_n$ to the LB operator of the underlying manifold.
 \begin{theorem}[Theorem 2.1 \cite{belkin2006convergence}]
 \label{thm:convergence}
 Let $X=\{x_1, x_2,...x_n\}$ be a set of $n$ points sampled i.i.d. from a $d$-dimensional manifold $\ccalM \subset \reals^N$. 
 Let $\bbG_n$ be a graph approximation of $\ccalM$ constructed from $X$ with weight values set as \eqref{eqn:weight} with $t_n = n^{-1/(d+2+\alpha)}$ and $\alpha>0$. Let $\bbL_n$ be the graph Laplacian of $\bbG_n$ and $\ccalL$ be the Laplace-Beltrami operator of $\ccalM$. Let $\lambda_{i}^n$ be the $i$-th eigenvalue of $\bbL_n$ and $\bm\phi_{i}^n$ be the corresponding normalized eigenfunction. Let $\lambda_i$ and $\bm\phi_i$ be the corresponding eigenvalue and eigenfunction of $\ccalL$ respectively. Then, it holds that
\begin{equation}
\label{eqn:convergence_spectrum}
    \lim_{n\rightarrow \infty } \lambda_i^n = \lambda_i, \quad \lim_{n\rightarrow \infty} |\bm\phi^{n}_i(x_j) -  \bm\phi_i(x_j)|=0, j=1,2 \hdots,n
\end{equation}
where the limits are taken in probability.
 \end{theorem}

With the definitions of neural networks on graph $\bbG_n$ and manifold $\ccalM$, the output difference can be written as 
 \begin{align}
    \nonumber \|\bm\Phi(\bbH,\bbL_n,\bbP_nf)-\bbP_n \bm\Phi&(\bbH,\ccalL, f))\| = \left\| \sum_{q=1}^{F_L}\bbx_L^q-\sum_{q=1}^{F_L}\bbP_n f_L^q \right\|\\
     & \leq \sum_{q=1}^{F_L} \left\| \bbx_L^q- \bbP_n f_L^q \right\|.
 \end{align}
 By inserting the definitions, we have 
 \begin{align}
   \nonumber  &\left\| \bbx_l^p- \bbP_n f_l^p \right\|\\
     &=\left\| \sigma\left(\sum_{q=1}^{F_{l-1}} \bbh_l^{pq}(\bbL_n) \bbx_{l-1}^q \right) -\bbP_n \sigma\left(\sum_{q=1}^{F_{l-1}} \bbh_l^{pq}(\ccalL) f_{l-1}^q\right) \right\|
 \end{align}
 with $\bbx_0=\bbP_n f$ as the input of the first layer. With a normalized Lipschitz nonlinearity, we have
  \begin{align}
    \| \bbx_l^p - \bbP_n f_l^p & \| \leq \left\|  \sum_{q=1}^{F_{l-1}} \bbh_l^{pq}(\bbL_n) \bbx_{l-1}^q    - \bbP_n \sum_{q=1}^{F_{l-1}} \bbh_l^{pq}(\ccalL)  f_{l-1}^q\right\|\\
    & \leq \sum_{q=1}^{F_{l-1}} \left\|    \bbh_l^{pq}(\bbL_n) \bbx_{l-1}^q    - \bbP_n   \bbh_l^{pq}(\ccalL)  f_{l-1}^q\right\|
 \end{align}
 The difference can be further decomposed as
\begin{align}
   \nonumber   \|    \bbh_l^{pq}(\bbL_n) & \bbx_{l-1}^q    - \bbP_n   \bbh_l^{pq}(\ccalL)  f_{l-1}^q \| 
   \\ \nonumber&\leq \|
\bbh_l^{pq}(\bbL_n) \bbx_{l-1}^q  - \bbh_l^{pq}(\bbL_n) \bbP_n f_{l-1}^q \\ &\qquad +\bbh_l^{pq}(\bbL_n) \bbP_n f_{l-1}^q  - \bbP_n   \bbh_l^{pq}(\ccalL)  f_{l-1}^q
    \|\\\nonumber
   & \leq \left\|
    \bbh_l^{pq}(\bbL_n) \bbx_{l-1}^q  - \bbh_l^{pq}(\bbL_n) \bbP_n f_{l-1}^q
    \right\|
  \\ &\qquad +
    \left\|
    \bbh_l^{pq}(\bbL_n) \bbP_n f_{l-1}^q  - \bbP_n   \bbh_l^{pq}(\ccalL)  f_{l-1}^q
    \right\|
\end{align}
The first term can be bounded as $\| \bbx_{l-1}^q - \bbP_nf_{l-1}^q\|$ with the initial condition $\|\bbx_0 - \bbP_n f_0\|=0$. The second term can be denoted as $D_{l-1}^n$. With the iteration employed, we can have
\begin{align}
 \nonumber \|\bm\Phi(\bbH,\bbL_n,\bbP_n f) - \bbP_n \bm\Phi(\bbH,\ccalL,f)\| 
 \leq
 \sum_{l=0}^L \prod\limits_{l'=l}^L F_{l'} D_l^n.
 \end{align}
 Therefore, we can focus on the difference term $D_l^n$, we omit the feature and layer index to work on a general form.
  \begin{align}
    &\nonumber \|\bbh(\bbL_n)\bbP_n f - \bbP_n\bbh(\ccalL) f\|\\
    &= \left\| \sum_{i=1}^n \hat{h}(\lambda_i^n) \langle \bbP_nf,\bm\phi_i^n \rangle_{\bbG_n}\bm\phi_i^n - \sum_{i=1}^\infty \hat{h}(\lambda_i)\langle f,\bm\phi_i\rangle_{\ccalM} \bbP_n \bm\phi_i  \right\|
 \end{align}
 
 We decompose the $\alpha$-FDT filter function as $\hat{h}(\lambda)=h^{(0)}(\lambda)+\sum_{l\in\ccalK_m}h^{(l)}(\lambda)$ as equations \eqref{eqn:h0} and \eqref{eqn:hl} show. With the triangle inequality and $n > N_\alpha=\max_{i}\{\lambda_i\in[\Lambda_k(\alpha)]_{k\in\ccalK_s}\}$, we start by analyzing the output difference of $h^{(0)}(\lambda)$ as
 \begin{align}
    & \nonumber \left\| \sum_{i=1}^{N_\alpha} {h}^{(0)}(\lambda_i^n) \langle \bbP_nf,\bm\phi_i^n \rangle_{\bbG_n}\bm\phi_i^n - \sum_{i=1}^{N_\alpha} {h}^{(0)}(\lambda_i)\langle f,\bm\phi_i\rangle_{\ccalM} \bbP_n \bm\phi_i  \right\|
     \\ 
     &\nonumber \leq  \left\| \sum_{i=1}^{N_\alpha} \left({h}^{(0)}(\lambda_i^n)- {h}^{(0)}(\lambda_i) \right) \langle \bbP_nf,\bm\phi_i^n \rangle_{\bbG_n}\bm\phi_i^n \right\| \\
     &  +\left\| \sum_{i=1}^{N_\alpha} {h}^{(0)}(\lambda_i)\left( \langle \bbP_n f,\bm\phi_i^n \rangle_{\bbG_n} \bm\phi_i^n - \langle f,\bm\phi_i \rangle_{\ccalM} \bbP_n \bm\phi_i \right)  \right\|.\label{eqn:conv-1}
 \end{align}
 
 The first term in \eqref{eqn:conv-1} can be bounded by leveraging the $A_h$-Lipschitz continuity of the frequency response. From the convergence in probability stated in \eqref{eqn:convergence_spectrum}, we can claim that for each eigenvalue $\lambda_i \leq \lambda_{N_\alpha}$, for all $\epsilon_i>0$ and all $\delta_i>0$, there exists some $N_i$ such that for all $n>N_i$, we have
\begin{gather}
 \label{eqn:eigenvalue}   \mathbb{P}(|\lambda_i^n-\lambda_i|\leq \epsilon_i)\geq 1-\delta_i,
 \end{gather}
Letting $\epsilon_i < \epsilon$ with $\epsilon > 0$, with probability at least $\prod_{i=1}^M(1-\delta_i) := 1-\delta$, the first term is bounded as 
 
\begin{align}
   &\nonumber \left\| \sum_{i=1}^{N_\alpha} ({h}^{(0)}(\lambda_i^n) - {h}^{(0)}(\lambda_i)) \langle \bbP_n f,\bm\phi_i^n \rangle_{\bbG_n} \bm\phi_i^n  \right\|\\
   & \leq \sum_{i=1}^{N_\alpha} |{h}^{(0)}(\lambda_i^n)-{h}^{(0)}(\lambda_i)| |\langle \bbP_n f,\bm\phi_i^n \rangle_{\bbG_n}| \|\bm\phi_i^n\|\\
   &\leq \sum_{i=1}^{N_\alpha} A_h |\lambda_i^n-\lambda_i| \|\bbP_n f\| \|\bm\phi_i^n \|^2\leq N_s A_h\epsilon,
\end{align} 
for all $n>\max\{\max_i N_i, N_\alpha\} := N$.

The second term in \eqref{eqn:conv-1} can be bounded combined with the convergence of eigenfunctions in \eqref{eqn:eigenfunction} as
\begin{align}
  & \nonumber \Bigg\| \sum_{i=1}^{N_\alpha}{h}^{(0)}(\lambda_i)\left( \langle \bbP_nf,\bm\phi_i^n \rangle_{\bbG_n}\bm\phi_i^n - \langle f,\bm\phi_i \rangle_{\ccalM} \bbP_n \bm\phi_i\right)  \Bigg\|\\
   & \leq \nonumber \Bigg\|  \sum_{i=1}^{N_\alpha} {h}^{(0)}(\lambda_i)  \left(\langle \bbP_n f,\bm\phi_i^n\rangle_{\bbG_n}\bm\phi_i^n  - \langle \bbP_nf,\bm\phi_i^n \rangle_{\bbG_n} \bbP_n\bm\phi_i\right)\Bigg\|\\
   &\label{eqn:term1}+ \left\| \sum_{i=1}^{N_\alpha}  {h}^{(0)}(\lambda_i) \left(\langle \bbP_n f,\bm\phi_i^n\rangle_{\bbG_n} \bbP_n\bm\phi_i -\langle f,\bm\phi_i\rangle_\ccalM \bbP_n\bm\phi_i \right) \right\|
\end{align}
From the convergence stated in \eqref{eqn:convergence_spectrum}, we can claim that for some fixed eigenfunction $\bm\phi_i$,  for all $\epsilon_i>0$ and all $\delta_i>0$, there exists some $N_i$ such that for all $n>N_i$, we have
\begin{gather}
 \label{eqn:eigenfunction}    \mathbb{P}(|\bm\phi_i^n(x_j) - \bm\phi_i(x_j)|\leq \epsilon_i)\geq 1-\delta_i,\quad \forall \; x_j\in X .
 \end{gather}
 Therefore, letting $\epsilon_i < \epsilon$ with $\epsilon > 0$, with probability at least $\prod_{i=1}^M(1-\delta_i) := 1-\delta$, for all $n>\max\{\max_i N_i, N_\alpha\} := N$, the first term in \eqref{eqn:term1} can be bounded as
\begin{align}
& \nonumber \left\|  \sum_{i=1}^{N_\alpha} {h}^{(0)}(\lambda_i) \left(\langle \bbP_n f,\bm\phi_i^n\rangle_{\bbG_n}\bm\phi_i^n  - \langle \bbP_nf,\bm\phi_i^n \rangle_{\ccalM} \bbP_n\bm\phi_i\right)\right\|\\
& \qquad \qquad\leq \sum_{i=1}^{N_\alpha} \|\bbP_n f\|\|\bm\phi_i^n - \bbP_n\bm\phi_i\|\leq N_s \epsilon,
\end{align}
because the frequency response is non-amplifying as stated in Assumption \ref{ass:filter_function}. The last equation comes from the definition of norm in $L^2(\bbG_n)$.
The second term in \eqref{eqn:term1} can be written as
\begin{align}
     & \nonumber \Bigg\| \sum_{i=1}^{N_\alpha}  {h}^{(0)}(\lambda_i^n) (\langle \bbP_n f,\bm\phi_i^n\rangle_{\bbG_n}  \bbP_n\bm\phi_i -\langle f,\bm\phi_i\rangle_\ccalM \bbP_n\bm\phi_i ) \Bigg\| \\
   &\leq \sum_{i=1}^{N_\alpha} |{h}^{(0)}(\lambda_i^n)| \left|\langle \bbP_n f,\bm\phi_i^n\rangle_{\bbG_n}  -\langle f,\bm\phi_i\rangle_\ccalM\right|\|\bbP_n\bm\phi_i\|.
\end{align}
Because $\{x_1, x_2,\cdots,x_n\}$ is a set of uniform sampled points from $\ccalM$, based on Theorem 19 in \cite{von2008consistency} we can claim that there exists some $N$ such that for all $n>N$
\begin{equation}
   \mathbb{P}\left(\left|\langle \bbP_n f,\bm\phi_i^n\rangle_{\bbG_n}  -\langle f,\bm\phi_i\rangle_\ccalM\right|\leq\epsilon \right)\geq 1-\delta,
\end{equation}
for all $\epsilon>0$ and $\delta>0$. Taking into consider the boundedness of frequency response $|{h}^{(0)}(\lambda)|\leq 1$ and the bounded energy $\|\bbP_n\bm\phi_i\|$. Therefore, we have for all $\epsilon>0$ and $\delta>0$,
\begin{align}
&\nonumber  \mathbb{P}\left(\left\| \sum_{i=1}^{N_\alpha}  h^{(0)}(\lambda_i^n) \left(\langle \bbP_n f,\bm\phi_i^n\rangle_{\bbG_n}  -\langle f,\bm\phi_i\rangle_\ccalM \right)\bbP_n\bm\phi_i  \right\|\leq N_s \epsilon\right)
\\& \qquad \qquad\qquad\qquad\qquad\qquad\qquad\qquad\qquad\geq 1-\delta,
\end{align}
for all $n>N$.

Combining the above results, we can bound the output difference of $h^{(0)}(\lambda)$. Then we need to analyze the output difference of $h^{(l)}(\lambda)$ and bound this as
\begin{align}
    \nonumber &\left\| \bbP_n \bbh^{(l)}(\ccalL)f -\bbh^{(l)}(\bbL_n)\bbP_n f \right\| 
    \\& \leq \left\| (\hat{h}(C_l)+\delta)\bbP_n f - (\hat{h}(C_l)-\delta)\bbP_nf\right\| \leq 2\delta\|\bbP_nf\|,
\end{align}
where $\bbh^{(l)}(\ccalL)$ and $\bbh^{(l)}(\bbL_n)$ are filters with filter function $h^{(l)}(\lambda)$ on the LB operator $\ccalL$ and graph Laplacian $\bbL_n$ respectively.
Combining the filter functions, we can write
\begin{align}
   \nonumber &\|\bbP_n\bbh(\ccalL)f-\bbh(\bbL_n)\bbP_n f\|\\\nonumber &=
    \Bigg\|\bbP_n\bbh^{(0)}(\ccalL)f +\bbP_n\sum_{l\in\ccalK_m}\bbh^{(l)}(\ccalL)f -\\& \qquad \qquad \qquad \bbh^{(0)}(\bbL_n)\bbP_n f - \sum_{l\in\ccalK_m} \bbh^{(l)}(\bbL_n)\bbP f \Bigg\|\\
    &\nonumber \leq \|\bbP_n \bbh^{(0)}(\ccalL)f-\bbh^{(0)}(\bbL_n)\bbP_n f\|+\\
    &\qquad \qquad \qquad \sum_{l\in\ccalK_m}\|\bbP_n \bbh^{(l)}(\ccalL)f-\bbh^{(l)}(\bbL_n)\bbP_nf\|.
\end{align}

Above all, we can claim that there exists some $N$, such that for all $n>N$, for all $\epsilon'>0$ and $\delta>0$, we have
\begin{equation}
    \mathbb{P}(\|\bbh(\bbL_n)\bbP_n f - \bbP_n\bbh(\ccalL) f\|\leq \epsilon')\geq 1-\delta.
\end{equation}

With $\lim\limits_{n\rightarrow \infty}D_l^n=0$ in high probability, this concludes the proof.

}

%

\end{document}